\documentclass[11pt]{article}

\usepackage[preprint]{acl}

\usepackage{times}
\usepackage{latexsym}

\usepackage{amsmath,amssymb,amsfonts}
\usepackage{algorithmic}
\usepackage{graphicx}
\usepackage{textcomp}
\usepackage{xcolor}
\usepackage{multirow} 
\usepackage{booktabs} 
\usepackage{tcolorbox}
\usepackage{colortbl}
\usepackage{amsmath}
\usepackage{tikz}
\usepackage{bm}
\usepackage{makecell}

\usepackage{booktabs}

\usepackage{enumerate}

\usepackage{amsthm}

\newtheorem{theorem}{Theorem}

\usepackage{color}
\usepackage{booktabs}
\usepackage[normalem]{ulem}
\useunder{\uline}{\ul}{}

\usepackage{subcaption}

\usepackage{url}
\usepackage[table]{xcolor}  

\usepackage{multirow}

\usepackage[T1]{fontenc}

\usepackage[utf8]{inputenc}

\usepackage{microtype}

\usepackage{inconsolata}

\usepackage{graphicx}

\title{Learning Invariant Modality Representation for Robust Multimodal Learning from a Causal Inference Perspective}

\author{
 \textbf{Sijie Mai\thanks{Corresponding Author} } \ \ 
 \textbf{Shiqin Han}
\\
 School of Computer Science, South China Normal University
\\
  \texttt{\{sijiemai,20222121019\}@m.scnu.edu.cn} \\
}

\begin{document}
\maketitle
\begin{abstract}

Multimodal affective computing aims to predict humans' sentiment, emotion, intention, and opinion using language, acoustic, and visual modalities. However, current models often learn spurious  correlations that harm generalization under distribution shifts or noisy modalities. To address this, we propose a causal modality-invariant representation (CmIR) learning framework for robust multimodal learning. At its core, we introduce a theoretically grounded disentanglement method that separates each modality into `causal invariant representation' and `environment-specific spurious representation' from a causal inference perspective. CmIR ensures that the learned invariant representations retain stable predictive relationships with labels across different environments while preserving sufficient information from the raw inputs via invariance constraint, mutual information constraint, and reconstruction constraint. Experiments across multiple multimodal benchmarks demonstrate that CmIR achieves state-of-the-art performance. CmIR particularly excels on out‑of‑distribution data and noisy data, confirming its robustness and generalizability. 
\end{abstract}

\section{Introduction}

Multimodal affective computing (MAC) aims to integrate information from language, acoustic, and visual modalities to predict high-level semantics such as sentiment, emotion, intention, and opinion \cite{mai2025learningbycomparing,Poria2017A}. 
Recently, MAC has achieved remarkable progress 
and become increasingly important for applications such as human-computer interaction, customer service automation, and affective computing systems.

\begin{figure}
    \centering
    \includegraphics[width=0.95\linewidth]{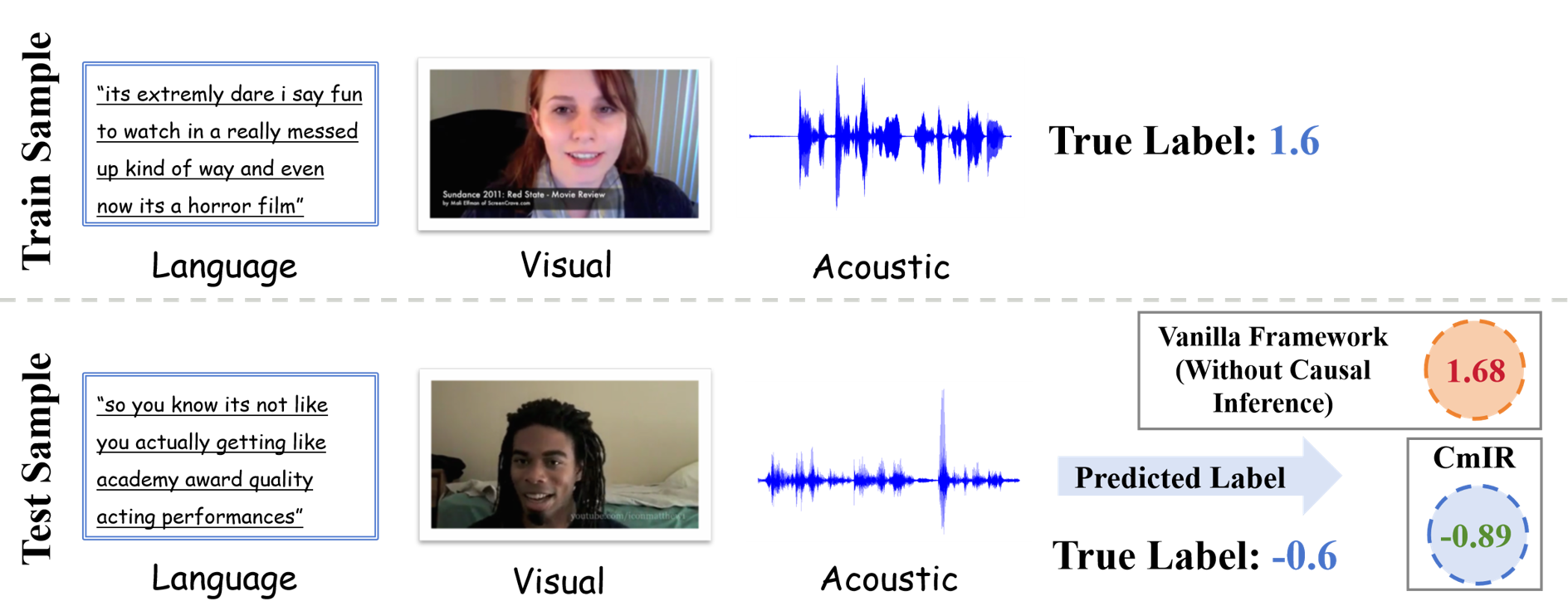}
    \caption{A case study on CMU-MOSI.  Vanilla model without causal inference makes incorrect prediction for the test sample where the speaker delivers negative comment while smiling, while CmIR accurately predicts the label based on correct causal relationships. }
    \label{fig:intro}
    \vspace{-0.2cm}
\end{figure}

Despite remarkable advances, 
existing approaches often learn spurious cross-modal correlations from training data 
rather than genuine causal relationships, harming model generalization when test data distributions differs from training distributions under distribution shifts or noisy modality conditions \cite{zhuang2025tmdc,peters2016causal}. This limitation restricts their practical deployment in real-world scenarios where environmental conditions, speaking styles, lighting conditions, and background contexts constantly vary \cite{arjovsky2019invariant}.
For instance, 
models might over-rely on the speaker’s consistent smile (a spurious visual feature) in training data rather than focusing on semantic content and genuine emotional expressions, leading to performance degradation when tested on different speakers (see Figure~\ref{fig:intro}). Similarly, noisy modalities (e.g., background noise or low-resolution visual frames) further disrupt spurious correlations, exacerbating the generalization gap. 
To address this limitation, some approaches focus on domain adaptation  to learn shared representations across domains \cite{domain1,zhu2025multimodal,domain2}, while others employ disentanglement strategy to separate different factors in multimodal data \cite{MISA,zhuang2025tmdc}. However, these methods lack causal interpretation and cannot guarantee the disentangled/learned features align with causal and spurious components.
Recent works incorporate causal inference principles to identify and eliminate spurious correlations \cite{xu2025debiased,MMCI2025,CLUE2022}. However, they often lack rigorous theoretical analysis or focus on specific biases (such as speaker bias and modality bias) rather than providing a general framework, which rely on prior knowledge or assumptions about the biases and may not generalize well to unseen environments or other types of distribution shifts.

To resolve these issues, we propose a Causal modality-Invariant Representation (CmIR) learning framework for robust multimodal learning. Drawing on causal inference principles \cite{arjovsky2019invariant,peters2016causal}, CmIR is theoretically grounded in causal inference and information theory, with a novel  feature disentanglement method that separates each modality into two complementary components: (1) \textbf{invariant causal representation} (\(Z_m^{inv}\)), which carries stable causal relationships with  labels 
across different environments; and (2) \textbf{environment-specific spurious representation} (\(Z_m^{spu}\)), which captures non-causal, environment-dependent noise and has no causal link to the label. The core insight is that while the distribution of raw features may vary across environments, the causal mechanisms between invariant features and prediction targets remain stable. CmIR is achieved through an elegant objective function that incorporates: (1) an \textbf{invariance constraint} to ensure invariant representations maintain stable relationships with labels across environments; (2) a \textbf{mutual information constraint} to minimize correlations between invariant and spurious components; and (3) a \textbf{reconstruction constraint} to preserve sufficient information from raw inputs. 
These constraints guarantee that invariant representations satisfy critical properties: environmental independence and causal sufficiency for prediction.
We then use invariant modality representations for prediction, ensuring robustness to distribution shifts and noisy modalities.
Compared to previous methods \cite{xu2025debiased,yang2024towards}, CmIR does not rely on specific bias or assumption. It directly learns invariant representations that are stable across all environments, 
which is more general and applicable to various distribution shifts.

Our contributions are summarized as follows:

\begin{itemize}
    \item \textbf{Methodological innovation:} We propose a novel framework name CmIR, which, for the first time, systematically disentangles each modality into causal and spurious components in MAC to comprehensively learn causality-sufficient invariant representations. 
    \item \textbf{Empirical validation:} Extensive experiments on multiple multimodal tasks show that CmIR achieves state-of-the-art results in both standard and \textbf{out-of-distribution} (OOD) benchmarks. Moreover, it exhibits superior robustness under \textbf{noisy modality} testing. 
    \item \textbf{Theoretical guarantees:} 
    We prove the existence and extractability of invariant representations given multi-environment training data. We also show that predictors based on invariant representations achieve lower worst-case OOD risk than those using raw features. 
\end{itemize}



\section{Related Work}
\label{sec:relatedwork}



\subsection{Multimodal Affective Computing}
Most works for MAC center on devising fusion techniques 
to learn discriminative multimodal representations \cite{Zadeh2018Memory,wang2025dlf,Zadeh2017Tensor} or employing techniques such as information bottleneck to regularize unimodal distributions \cite{nll,MIB,luo2025triagedmsa}.
Recently, multimodal large language models have enabled the direct processing of multimodal signals using large pre-trained models, enhancing the interpretation of human affective states \cite{zhao2025humanomnilargevisionspeechlanguage,xu2025qwen2}.
However, these methods often neglect to improve the generalizability of models for OOD data.
To enhance generalization and robustness, some methods focus on domain adaptation to learn shared representations across domains \cite{domain1,zhu2025multimodal,domain2}, while others employ disentangled learning to separate different factors in multimodal data \cite{confede,MISA,zhuang2025tmdc,MFM}. However, these methods lack causal interpretation and cannot guarantee the disentangled/learned features align with causal and spurious components.

\subsection{Causal Inference}

Causal inference can detect and remove non-causal associations in complex datasets to improve model robustness and generalization \cite{wang2022causal,niu2021counterfactual}.
Many causality-based methods have been introduced to reduce cross-modal bias in multimodal learning. 
Researchers employ counterfactual reasoning to refine attention distributions~\cite{huang2025atcaf}, apply front-door and back-door adjustments to decouple spurious links between text and vision~\cite{liu2023cross},  develop counterfactual and debiasing frameworks~\cite{CLUE2022,MulDeF2024,sun2023general}, and design causal intervention modules to separate misleading connections between expressive style and semantic content~\cite{xu2025debiased}. 
However, most methods are restricted to single or specific modality pairs, or they rely on explicitly annotated bias types that require 
domain knowledge~\cite{nam2020learning}. This limitation hinders their broader application to complex multimodal data where biases are often implicit and not predefined.
In contrast, we propose a general invariant representation learning framework without requiring predefined bias types. 
Moreover,
Invariant Risk Minimization \cite{arjovsky2019invariant} and Invariant Causal Mechanism of CLIP \cite{songlearning} that aim to learn invariant features are related to our work, but they focus on single modalities or specific modality combinations in particular application scenarios. In contrast, we provide a more general multimodal causal framework with feature disentanglement to understand and learn the properties of invariant features more comprehensively and accurately.

\section{Theoretical Analysis}
\label{sec:theory}

Here we establish a theoretical foundation for our causal approach to multimodal learning. We first establish the existence and extractability of causal invariant modality representations, then prove their advantages in terms of generalization performance.

\begin{figure}
    \centering
    \includegraphics[width=0.95\linewidth]{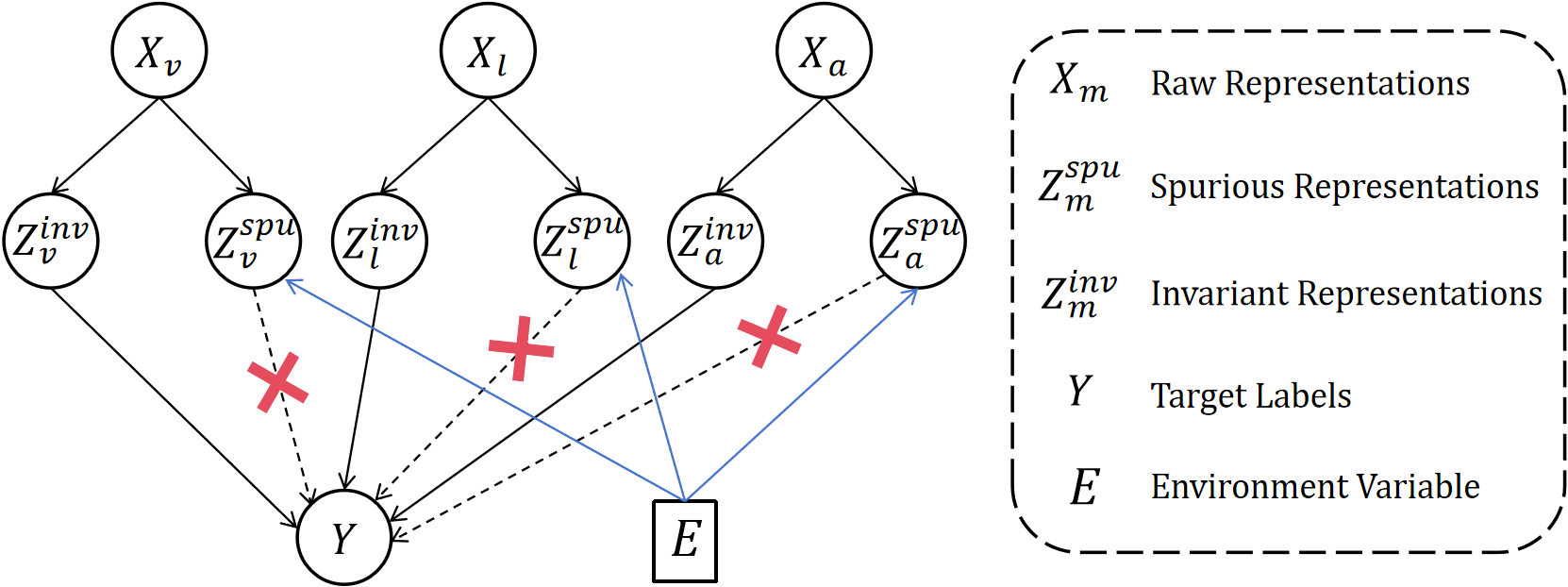}
    \caption{The SCM of CmIR for prediction process. It includes language ($l$), visual ($v$), acoustic ($a$) modalities.}
    \label{fig:causal_relation}
\end{figure}

\subsection{Definition of Invariant Representations}
\label{subsec:invariant_representations}

We begin by formalizing the causal structure of multimodal learning.  Consider $M$ modalities $\{X_m\}_{m=1}^M$ and prediction target $Y$. Following the structural causal model (SCM) framework~\cite{peters2016causal}, we assume the data generation process involves an environment variable $E$ that induces distribution shifts, and each modality $X_m$ can be decomposed into an invariant component $Z_m^{\text{inv}}$ (containing causal features) and a spurious component $Z_m^{\text{spu}}$ (containing environment-specific features). As shown in Figure~\ref{fig:causal_relation}, we assume $Y$ is causally influenced only by invariant components $\{Z_m^{\text{inv}}\}$ and is independent of $E$ and $\{Z_m^{\text{spu}}\}$.

\begin{theorem}[Definition of Causal Invariant Modality Representations]
\label{thm:invariant_representations}
Assume there exists a function class $\Phi_m = \{\phi_m: \mathcal{X}_m \rightarrow \mathcal{Z}_m^{\text{inv}}\}$ and a distribution distance measure $\mathcal{D}$ 
such that the following optimization problem has a solution:
\begin{equation*}
\setlength{\abovedisplayskip}{3pt}
\setlength{\belowdisplayskip}{3pt}
\begin{split}
\phi_m^*\! =\!  \arg\min_{\phi_m \in \Phi_m} \max_{e_1,e_2 \in \mathcal{E}} 
&\mathcal{D}(P(Y|\phi_m(X_m),E\!=\!e_1), \\
&P(Y|\phi_m(X_m),E=e_2))
\end{split}
\end{equation*}
Then $\phi_m^*(X_m)$ constitutes a causal invariant modality representation satisfying:
\begin{equation*}
\setlength{\abovedisplayskip}{3pt}
\setlength{\belowdisplayskip}{3pt}
\begin{split}
P(Y|\phi_m^*(X_m), E\!=\!e_1) =& P(Y|\phi_m^*(X_m), E\!=\!e_2), \\ 
& \forall e_1,e_2 \in \mathcal{E}
\end{split}
\end{equation*}
\end{theorem}

\begin{proof} 
See Section~\ref{subsec:invariant_representations_a} for the proof.
\end{proof}

Theorem~\ref{thm:invariant_representations} suggests that $\phi_m^*(X_m)$ is a valid invariant modality representation capturing only causal features for robust prediction. 
If we can learn a representation $\phi(X_m)$ such that the conditional distribution of the label $Y$ given this representation is the same across all environments (i.e., $P(Y|\phi(X_m),E=e)$ does not depend on $e$), then this representation must capture only the causal features (i.e., causal features exist). Intuitively, causal relationships between features and labels are invariant under changes of the environment. If the prediction rule based on $\phi(X_m)$ remains unchanged when the environment varies, it means that $\phi(X_m)$ does not contain any environment-specific spurious information. In other words, it 
blocks all backdoor paths from environment $E$ to label $Y$. Thus, the invariance condition is a signature of causality.

\subsection{Extraction of Invariant Representations}
\label{subsec:disentangled_representations}

While Theorem~\ref{thm:invariant_representations} establishes the definition of invariant representations, practical implementation requires extracting these representations from raw modalities while preserving all relevant information. This motivates our disentanglement approach.

\begin{theorem}[Theoretical Guarantee for Extracting Disentangled Representations]
\label{thm:disentangled_representations}
Consider encoder functions $g_m: \mathcal{X}_m \rightarrow (\mathcal{Z}_m^{\text{inv}}, \mathcal{Z}_m^{\text{spu}})$ and decoder functions $r_m: (\mathcal{Z}_m^{\text{inv}}, \mathcal{Z}_m^{\text{spu}}) \rightarrow \mathcal{X}_m$ that optimize the following objective:
\begin{align*}
\setlength{\abovedisplayskip}{3pt}
\setlength{\belowdisplayskip}{3pt}
&\min_{\{g_m,r_m,h\}_{m=1}^M} \mathbb{E}_{e\in\mathcal{E}}[\mathcal{L}_{\text{pred}}(Y, h(\{Z_m^{\text{inv}}\}_{m=1}^M))] \\
&+ \lambda_1 \sum_{m=1}^M \mathcal{R}_{\text{inv}}^{(m)} 
+ \lambda_2 \sum_{m=1}^M \mathcal{R}_{\text{dec}}^{(m)} 
+ \lambda_3 \sum_{m=1}^M \mathcal{R}_{\text{rec}}^{(m)}
\end{align*}
where $\mathcal{L}_{\text{pred}}$ is the task prediction loss, $h$ is the prediction head, $\mathcal{R}_{\text{inv}}^{(m)} = \sum_{e_1 \in \mathcal{E}} \sum_{e_2 \in \mathcal{E}} 
    \mathcal{D}(P(Y|Z_m^{\text{inv}},E=e_1), 
    P(Y|Z_m^{\text{inv}},E=e_2))$ enforces invariance,
$\mathcal{R}_{\text{dec}}^{(m)} = I(Z_m^{\text{inv}}; Z_m^{\text{spu}})$ minimizes mutual information between invariant and spurious components,
$\mathcal{R}_{\text{rec}}^{(m)} = \|X_m - r_m(Z_m^{\text{inv}}, Z_m^{\text{spu}})\|^2$ ensures reconstruction capability,
and $\lambda_1,\lambda_2,\lambda_3 > 0$ are hyperparameters.
Assuming the label $Y$ is independent of environment $E$, the function classes $\{g_m,r_m,h\}$ have sufficient capacity and the data follows the SCM described in Section~\ref{subsec:invariant_representations}, then as $\lambda_1,\lambda_2,\lambda_3 \rightarrow \infty$, the optimal solution satisfies:
\begin{enumerate}
\item \(\displaystyle \lim_{\lambda_3\to\infty} \mathbb{E}\bigl[\|X_m - r_m(Z_m^{\mathrm{inv}},Z_m^{\mathrm{spu}})\|^2\bigr] = 0\) (perfect reconstruction is achieved)
    \item $Z_m^{\text{inv}} \perp\!\!\!\perp E$ (invariant component is environment-independent)
    \item $I(Y; Z_m^{\text{spu}} | Z_m^{\text{inv}}, E) = 0$ (spurious component contains no additional causal information)
\end{enumerate}
\end{theorem}

\begin{proof}
See Section~\ref{subsec:disentangled_representations_a} for the proof.
\end{proof}

Theorem~\ref{thm:disentangled_representations} suggests that causal invariant representations can be learned using our CmIR.
The invariance constraint forces $Z^{inv}$ to have the same predictive relationship with $Y$ across environments, making it environment-independent (capturing only causal features). The reconstruction loss ensures that the pair $(Z^{inv}, Z^{spu})$ retains all information from the original input, preventing information loss and avoiding degenerate decomposition solutions. The mutual information minimization pushes $Z^{inv}$ and $Z^{spu}$ to be statistically independent, so that $Z^{spu}$ cannot carry any causal information about $Y$ that is already in $Z^{inv}$. Reconstruction loss and mutual information minimization together force $Z^{inv}$ to contain all causal information in the original input. These three constraints lead to a clean decomposition: $Z^{inv}$ contains only causal factors and encompasses all causal information from the original input, and $Z^{spu}$ contains only environment-specific noise.

\subsection{Distributionally Robust Risk Advantage of Invariant Representations}
\label{thm3}

Having established how to learn invariant representations, we now prove their theoretical advantages for worst-case OOD risk under distribution shift.

\begin{theorem}[Distributionally Robust Risk Advantage of Invariant Representations]
\label{thm:ood_risk_advantage}
Let $\mathcal{H}$ be a hypothesis class over $\mathcal{X} \times \mathcal{Y}$, where $\mathcal{X} = \mathcal{X}_1 \times \mathcal{X}_2 \times \dots \times \mathcal{X}_M$ is the $M$-modal feature space and $\mathcal{Y}$ is the label space. Let $\mathcal{E}_{\text{all}}$ denote the set of all possible environments, each corresponding to a distribution $P^e(x, y)$. Let $h_{\text{inv}} \in \mathcal{H}$ be a predictor using invariant representations $Z^{inv} = \{Z_m^{\text{inv}}\}_{m=1}^M$, and $h_{\text{raw}} \in \mathcal{H}$ be a predictor using raw multimodal representations $X = \{X_m\}_{m=1}^M$. Assume:

1. \textbf{Invariance Condition}: The invariant representations satisfy $P^e(Y|Z^{\text{inv}}) = P^{e'}(Y|Z^{\text{inv}})$ for all $e,e' \in \mathcal{E}_{\text{all}}$.

2. \textbf{Information Sufficiency}: The mutual information between invariant representations and raw features satisfies $I(Z^{\text{inv}}; X) > c$ for some constant $c > 0$ ($Z^{\text{inv}}$ contains enough information from $X$).

3. \textbf{Loss Function Regularity}: The loss function $\ell$ is $L$-Lipschitz continuous and bounded.

Then the worst-case OOD risk satisfies:
\begin{equation*}
R^{\text{OOD}}(h_{\text{inv}}) < R^{\text{OOD}}(h_{\text{raw}})
\end{equation*}
where $R^{\text{OOD}}(h) = \max_{e \in \mathcal{E}_{\text{test}}} R^e(h)$ and $R^e(h) = \mathbb{E}_{(x,y)\sim P^e}[\ell(h(x), y)]$.
\end{theorem}

\begin{proof}
See Section~\ref{thm3_a} for the proof.
\end{proof}

Theorem~\ref{thm:ood_risk_advantage} proves that under realistic conditions, predictors based on invariant representations achieve strictly lower worst-case out-of-distribution risk than those using raw features. The intuition is straightforward: raw features contain both causal and spurious parts. The spurious part may change arbitrarily in new environments, causing large errors in the worst-case scenario. In contrast, the invariant representation relies only on the stable causal mechanism, which remains unchanged across environments, thereby guaranteeing more reliable performance even under the most adverse distribution shifts.

\begin{figure*}
    \centering
    \includegraphics[width=0.95\linewidth]{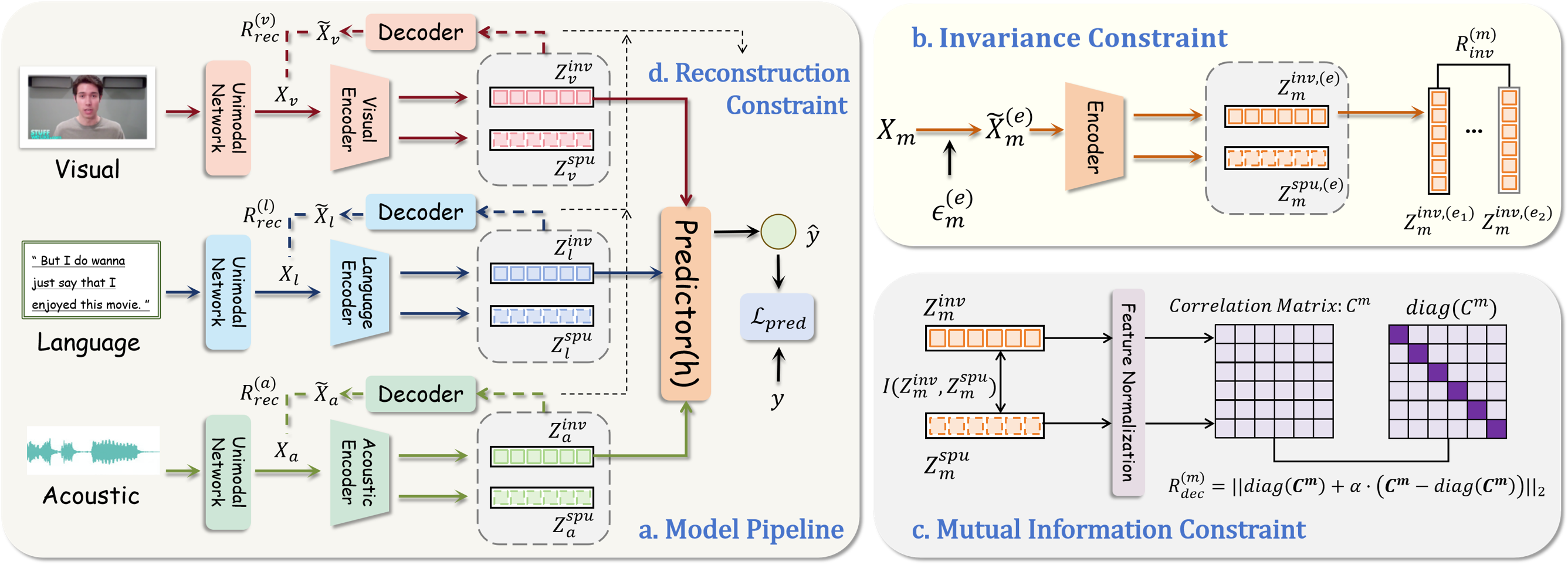}
    \caption{The overall framework of CmIR and the visualization of the proposed constraints.}
    \label{fig:framework}
    \vspace{-0.1cm}
\end{figure*}

\section{Algorithm Implementation}

Here we elaborate on the implementation of CmIR proposed in Theorem~\ref{thm:disentangled_representations}. Our objective is to learn the disentangled modality representations, 
and perform prediction by fusing invariant representations $Z_m^{\text{inv}}$. The overall framework (see Figure~\ref{fig:framework}) consists of unimodal networks that produces raw unimodal features $X_m \in \mathbb{R}^{1\times d} $ (see Appendix~\ref{unimodal_a} for unimodal networks), encoders $g_m$, decoders $r_m$, and a prediction head (predictor) $h$. 
Next, we detail the implementation of each loss component.

\subsection{Prediction Loss $\mathcal{L}_{\text{pred}}$}

The prediction loss ensures that invariant representations effectively predict the target label $Y$. 
Given a batch of data, the prediction loss is computed as:
\begin{equation}
Z_{m}^{\text{inv}},\ \ Z_{m}^{\text{spu}} = g_m(X_m)
\end{equation}
\begin{equation}
\mathcal{L}_{\text{pred}} = \frac{1}{N} \sum_{i=1}^{N} \ell\left(h\left(\{Z_{m,i}^{\text{inv}}\}_{m=1}^{M}\right), Y_i\right)
\end{equation}
where $g_m$ is the encoder for modality $m$, 
$N$ is the batch size,  and $\ell$ is the corresponding loss function.
For classification tasks (e.g., humor detection and sarcasm detection), cross-entropy loss is utilized. For regression tasks (e.g., sentiment analysis), mean squared error (MSE) or mean absolute error (MAE) is used.
The predictor is implemented using unimodal feature concatenation and a few multi-layer perception layers (see Figure~\ref{structure}).

\subsection{The Invariance Constraint $\mathcal{R}_{\text{inv}}^{(m)}$}
\label{sec:invariant}
The invariance constraint requires that $P(Y|Z_m^{\text{inv}}, E)$ remains invariant across different environments. However, real-world data often lack explicit environment labels $E$. To address this, we draw inspiration from data augmentation and simulate different \textbf{virtual environments} by injecting varying degrees of noise into the raw features $X_m$.
For each sample, we assign a random virtual environment label $e \in \{1, 2, ..., K\}$, where $K$ is the number of environments (a hyperparameter, see Table~\ref{t2223}). Then we perform \textbf{noise perturbation} via applying an additive noise $\epsilon_m^{(e)} = \alpha^{(e)} \cdot \epsilon_m$ to $X_m$ based on the environment label $e$:
    \begin{equation}
    \label{eq_inva}
    \tilde{X}_m^{(e)} = X_m + \alpha^{(e)} \cdot \epsilon_m, \quad \epsilon_m \sim \mathcal{N}(0, \Sigma_m)
    \end{equation}
    where $\alpha^{(e)}$ is an environment-dependent coefficient controlling the noise intensity, and $\Sigma_m$ is the modality-specific covariance matrix (which can be set as the identity matrix or estimated from the data). The noise coefficient $\alpha^{(e)}$ is different for different environments, which is defined as $\alpha^{(e)} = \alpha^{(1)} * e, e \in \{1, 2, ..., K \} $. $ \alpha^{(1)}$ is the noise coefficient for Environment $1$, which is a hyperparameter whose values are shown in Table~\ref{t2223}.
    The perturbed feature $\tilde{X}_m^{(e)}$ is fed into the encoder $g_m$ to obtain the invariant representation $Z_m^{\text{inv},(e)}$ for that environment.
    Finally, the invariance constraint is implemented by minimizing the discrepancy between the conditional distributions $P(Y|Z_m^{\text{inv}}, E)$ across different environments. To realize this, we can adopt a common strategy that 
    enforces consistency in the output distributions of predictor across environments for classification tasks. Specifically, Kullback-Leibler (KL) divergence can be used as the distribution distance measure $\mathcal{D}$:
    \begin{equation*}
    \mathcal{R}_{\text{inv}}^{(m)}\!\! =\!\! \sum_{e_1 \neq e_2}\!\text{ KL}\!\left(\! P(Y|Z_m^{\text{inv},(e_1)}) \| P(Y|Z_m^{\text{inv},(e_2)})\! \right)
    \end{equation*}
    where $P(Y|Z_m^{\text{inv},(e)})$ is given by the output of predictor (after Softmax) on the corresponding environment's representation. However, this strategy requires to implement a unimodal predictor for each invariant modality representation which increases the model complexity, and training noise might be introduced if unimodal predictors are not well trained. Moreover, it is hard for regression tasks to calculate the KL-divergence between output distributions. To this end, we adopt a simpler implementation that encourages the learned invariant representations to be identical across different environments, which is a stronger constraint that satisfies the invariance constraint because the output distributions must be the same if input features were the same. We have provided the comparison results of these two variants in  Appendix~\ref{sec:comparison_invariant}. Specifically, the invariance constraint is implemented as: 
\begin{equation}
    \mathcal{R}_{\text{inv}}^{(m)} = \sum_{e_1 \neq e_2} \| Z_m^{\text{inv},(e_1)} - Z_m^{\text{inv},(e_2)} \|_1
    \end{equation}
Minimizing this term encourages the model to extract features from $X_m$ that are insensitive to noise perturbations (simulating environmental changes).
In practice, for each sample in a batch, we can assign an environment label $e$ and generate $\tilde{X}_m^{(e)}$ using Eq.~\ref{eq_inva}. For $K$ environments, we can generate $K+1$ variants of unimodal representations (including the original unimodal representation itself). The invariance loss $\mathcal{R}_{\mathrm{inv}}^{(m)}$ is computed over all $K(K+1)/2$ pairs for each sample in the batch, ensuring strong invariance constraints. 

\subsection{Mutual Information Constraint $\mathcal{R}_{\text{dec}}^{(m)}$}

Theorem~2 requires minimizing the mutual information $I(Z_m^{\text{inv}}; Z_m^{\text{spu}})$ between the invariant representation $Z_m^{\text{inv}}$ and the spurious representation $Z_m^{\text{spu}}$ to promote their disentanglement and capture independent information. Directly computing mutual information is intractable. We employ a widely used and effective alternative: approximating mutual information minimization by enforcing \textbf{orthogonality} (zero linear correlation) between the two representations in the feature space, which is a practical and computationally efficient proxy for minimizing mutual information between $Z_m^{inv}$ and $Z_m^{spu}$. Orthogonality is a necessary condition for statistical independence, and we augment this constraint with invariance and reconstruction constraints to ensure semantic separation of causal and spurious factors. This proxy is widely used in disentanglement learning for its scalability to large multimodal datasets.
Minimizing this term encourages $Z_m^{\text{inv}}$ and $Z_m^{\text{spu}}$ to learn in orthogonal directions, thereby reducing information redundancy between them.

Specifically, for each batch of training data, the correlation matrix $\bm{C}^{m}$ can be calculated as:
\begin{equation}
\bm{C}^{m} =  Nor(\bm{Z}_m^{\text{inv}}) Nor(\bm{Z}_m^{\text{spu}})^\top 
\end{equation}
where $Nor(x) = \frac{x - mean(x)}{std(x)}$ denotes feature normalization, $\bm{Z}_m^{\text{inv}}\in \mathbb{R}^{N \times d}$ denotes a batch of invariant modality representations, and $\bm{C}^{m} \in \mathbb{R}^{N \times N} $ is the correlation matrix for modality $m$. Then, we enforce orthogonality via the following operation:
\begin{equation}
\label{eq66}
\mathcal{R}_{\text{dec}}^{(m)}\!\!=\! \| diag(\bm{C}^{m})\! +\! \alpha  \cdot  (\bm{C}^{m}\! -\! diag(\bm{C}^{m})) \|_2
\end{equation}
where $diag(\bm{C}^{m})$ denotes the diagonal matrix of $\bm{C}^{m}$, and $\alpha$ is a hyperparameter that is between zero and one. 
In Eq.~\ref{eq66}, we use the Frobenius norm of matrix $\|\cdot\|_F$ (standard for matrix regularization). The term $\alpha$ balances the constraint strength between diagonal (same-sample) and off-diagonal (cross-sample) terms in the correlation matrix, which is a hyperparameter that depends on datasets (see Table~\ref{t2223}). When $\alpha$ is less than 1, it can down-weight off-diagonal terms to focus on sample-wise orthogonality. 
In this way, we can enforce a stricter constraint on the invariant and spurious representations from the same sample, and also encourage invariant and spurious representations from different samples to be orthogonal, promoting the statistical independence between two representations.


\subsection{Reconstruction Constraint $\mathcal{R}_{\text{rec}}^{(m)}$}

The reconstruction loss ensures that the disentangled representations $(Z_m^{\text{inv}}, Z_m^{\text{spu}})$ retain all information from the original input $X_m$, preventing the loss of crucial content during representation learning.

Firstly, the encoder $g_m$ maps the input $X_m$ to the disentangled representation pair $(Z_m^{\text{inv}}, Z_m^{\text{spu}})$. The decoder $r_m$ then attempts to reconstruct the original input from this pair:
\begin{equation}
\hat{X}_m = r_m(Z_m^{\text{inv}}, Z_m^{\text{spu}})
\end{equation}
The reconstruction loss is computed using MSE:
\begin{equation}
\mathcal{R}_{\text{rec}}^{(m)} = \| X_m - \hat{X}_m \|_2^2
\end{equation}
The encoder and decoder are implemented as multi-layer perceptron networks (see Figure~\ref{structure}). 

\subsection{Overall Optimization Objective}

The complete optimizable objective function is:
\begin{equation}
\mathcal{L} \! = \! \mathcal{L}_{\text{pred}}\! +\! \sum_{m=1}^{M} \lambda_1   \mathcal{R}_{\text{inv}}^{(m)}\! +\! \lambda_2  \mathcal{R}_{\text{dec}}^{(m)} + \lambda_3  \mathcal{R}_{\text{rec}}^{(m)}
\end{equation}
where $\lambda_1, \lambda_2, \lambda_3$ are hyperparameters that balance the importance of each constraint. 

\section{Experiments}\label{sec:Experiments}

CmIR is evaluated on multiple tasks of MAC, including multimodal sentiment analysis (MSA), multimodal humor detection (MHD) and multimodal sarcasm detection (MSD). The used datasets include
CMU-MOSI \cite{zadeh2016multimodal}, CMU-MOSI (OOD) \cite{CLUE2022}, CMU-MOSEI \cite{MOSEI}, CH-SIMS-v2 \cite{simsv2}, UR-FUNNY \cite{ur_funny} and MUStARD \cite{msd}. Due to space limitation, \textbf{we introduce experimental settings, baselines, datasets and additional results in Appendix.}
Codes are available at: \url{https://github.com/TmacMai/CmIR}.

\begin{table*}
\centering
\small
 \vspace{-0.2cm}
 \caption{ \label{tbase}Comparisons on the CMU-MOSI and CMU-MOSEI datasets. The results labeled with $^{\dag}$ are obtained from original papers, and other results are obtained from our experiments.
 The best results are highlighted.
 }
 \vspace{-0.1cm}
 \resizebox{1.99\columnwidth}{!}{\begin{tabular}{c||c|c|c|c|c|c|c|c|c|c|c}
  \noalign{\hrule height 1pt} 
\rowcolor{lightgray!40}
   &     & \multicolumn{5}{c|}{CMU-MOSI} & \multicolumn{5}{c}{CMU-MOSEI}  \\
 \cline{3-12}
 \rowcolor{lightgray!40}
\multirow{-2}{*}{Model} & \multirow{-2}{*}{Venue} & Acc7$\uparrow$  & Acc2$\uparrow$ & F1$\uparrow$ & MAE$\downarrow$ & Corr$\uparrow$ & Acc7$\uparrow$  & Acc2$\uparrow$ & F1$\uparrow$ & MAE$\downarrow$ & Corr$\uparrow$\\
\hline
\hline
Self-MM \cite{mmsa} & AAAI 21 & 45.8 &   84.9  &   84.8  & 0.731  & 0.785  &  53.0 &   85.2  &   85.2  &  0.540 & 0.763 \\

ConFEDE~\cite{confede} & ACL 23 & 43.3	& 85.1	& 85.2	&0.728 	&0.784 & 52.7	&85.7	&85.6	&0.538 	& 0.772 \\ 

FDMER$^{\dag}$~\cite{yang2022disentangled_FDMER} & ACM MM 22 & 44.1	& 84.6 & 84.7	& 0.724	& 0.788 & 54.1	&86.1 & 85.8 & 0.536 & 0.773 \\ 


SuCI$^{\dag}$~\cite{xu2025debiased} & AAAI 25 & 42.2& 84.6	& 84.5&- 	&- & \underline{54.6}	&85.8	&85.7	& -	& - \\ 

   C-MIB \citep{MIB} & TMM 23 &  47.7 &   87.8 &  87.8  & 0.662  & 0.835  & 52.7 &  86.9  & 86.8 & 0.542  &  0.784\\
EMOE \cite{EMOE2025} & CVPR 25  & 45.2  & 84.8  &  84.8  &  0.723 & 0.790  & 52.5  & 85.0  &  85.0 & 0.542 & 0.760 \\
Multimodal Boosting  \citep{10224356} & TMM 24 & \underline{49.1}  &  \underline{88.5} & 88.4  &  \underline{0.634} & \underline{0.855}  & 54.0  &  86.5 & 86.5 & \underline{0.523}  & 0.779 \\
ITHP \cite{ithp} & ICLR 24 &  47.7 &  \underline{88.5} &   \underline{88.5}  & 0.663  & \textbf{0.856}  & 52.2  &  87.1  & 87.1 &  0.550 & 0.792 \\
Diffusion Bridge \cite{lee2025diffusion}  & CVPR 25  & 47.3 & 86.9 & 86.8 & 0.649 & 0.839 & 53.1 & 87.1 & 87.0 & 0.531 & \textbf{0.800} \\
GSCon \cite{shi2025gradient} & TIP 25 & 45.7  & 88.1  &  88.0  &  0.696 & 0.832  & 50.8  & \underline{87.4}  & \underline{87.4} & 0.561& 0.750 \\
    \hline
    \rowcolor[HTML]{EBFAFF}
    CmIR & -  & \textbf{49.8}  &   \textbf{89.6} &  \textbf{89.5}  &  \textbf{0.616} & 0.853  & \textbf{55.1} &  \textbf{87.8} &   \textbf{87.7} & \textbf{0.513}  &  \underline{0.793} \\
    \noalign{\hrule height 1pt} 
     \end{tabular}}
\end{table*}%

\begin{table}[t]
    \vspace{0.1cm}
    \caption{The results on CH-SIMS-v2. 
    }
    \vspace{-0.2cm}
    \label{tab:SIMSResult}
    \resizebox{0.5\textwidth}{!}{
        \begin{tabular}{c||c|c|c|c|c|c}
        \noalign{\hrule height 1pt} 
\rowcolor{lightgray!40}
         & \multicolumn{6}{c}{CH-SIMS v2} \\ 
        \cline{2-7}
        \rowcolor{lightgray!40}
       \multirow{-2}{*}{Model}  & Acc5↑ & Acc3↑ & Acc2↑ & F1↑ & MAE↓ & Corr↑  \\ \hline
       \hline
        MISA \cite{MISA} & 47.5 & 68.9 & 78.2 & 78.3 & 0.342 & 0.671 \\
        MAG-BERT \cite{MAG-BERT} & 49.2 & 70.6 & 77.1 & 77.1 & 0.346 & 0.641 \\
        Self-MM \cite{mmsa} & \underline{53.5} & 72.7 & 78.7 & 78.6 & 0.315 & 0.691 \\
        MMIM \cite{MMIM} & 50.5 & 70.4 & 77.8 & 77.8 & 0.339 & 0.641 \\
        AV-MC \cite{simsv2} & 52.1 & 73.2 & \underline{80.6} & \underline{80.7} & 0.301 & 0.721 \\
        KuDA \cite{kuda} & 53.1 & \underline{74.3} & 80.2 & 80.1 & \underline{0.289} & \underline{0.741} \\
        DLF \citep{wang2025dlf} & 47.5 & 70.0 & 78.1 & 77.9 & 0.346 & 0.683  \\
        Diffusion Bridge \cite{lee2025diffusion} & 52.5 & 70.7 & 78.6 & 79.0 & 0.323 & 0.677 \\
        \hline
        \rowcolor[HTML]{EBFAFF}
        CmIR & \textbf{56.0} & \textbf{75.0} & \textbf{81.9} & \textbf{82.0} & \textbf{0.286} & \textbf{0.742}  \\  
         \noalign{\hrule height 1pt} 
        \end{tabular}
    }
\end{table}

\subsection{Performance on the MSA Task}


The performance of CmIR on MSA is summarized in Table~\ref{tbase} and Table~\ref{tab:SIMSResult}. On CMU-MOSI, CmIR surpasses strong baseline ITHP \cite{ithp} by more than 2 points in Acc7 and  1 points in Acc2. For CMU-MOSEI, it outperforms GSCon \cite{shi2025gradient} and achieves the best scores in Acc2, F1, MAE, and Acc7.
Compared with feature disentanglement method FDMER \cite{yang2022disentangled_FDMER} and previous backdoor-adjustment work that focuses on specific confounders \cite{xu2025debiased}, CmIR demonstrates considerable improvement.
Similar superiority is observed on CH-SIMS-v2 (Table~\ref{tab:SIMSResult}), where CmIR outperforms all baselines across every metric, including an improvement of 2.5 points in Acc5. Overall, \textbf{CmIR establishes state-of-the-art results on MSA across three standard benchmarks}. This strong performance is primarily attributed to CmIR's causal learning strategy, which effectively learns invariant representations across all environments that eliminate general bias and enable a more robust multimodal learning.

\subsection{Performance on the MHD and MSD Tasks}


To assess the task generalizability of CmIR, we evaluate it on MHD and MSD  (classification tasks) using UR-FUNNY  and MUStARD datasets. 
The baselines include
MulT \cite{MULT}, Self-MM \cite{mmsa}, MMIM \cite{MMIM}, HKT \cite{HKT}, DMD \cite{li2023decoupled_cvpr}, DMD+SuCI \cite{xu2025debiased}, AtCAF \cite{huang2025atcaf}, MAG-XLNet \cite{MAG-BERT}, MCL \cite{mcl}, MGCL \cite{mgcl}, ITHP \cite{ithp}, MISA \cite{MISA}, and MIL \cite{MIL}, where DMD+SuCI and AtCAF are causality-based methods. As shown in Figure~\ref{mhd_msd}, CmIR surpasses the strongest baselines (AtCAF and MGCL) by margins exceeding 4 and 2 points on UR-FUNNY and MUStARD, respectively. Overall, \textbf{CmIR achieves competitive performance on both MHD and MSD}, confirming its effectiveness and \textbf{strong generalizability to diverse multimodal tasks}.

\begin{figure}
    \centering
    \begin{subfigure}[t]{0.5\linewidth} 
        \centering
        \includegraphics[width=\linewidth]{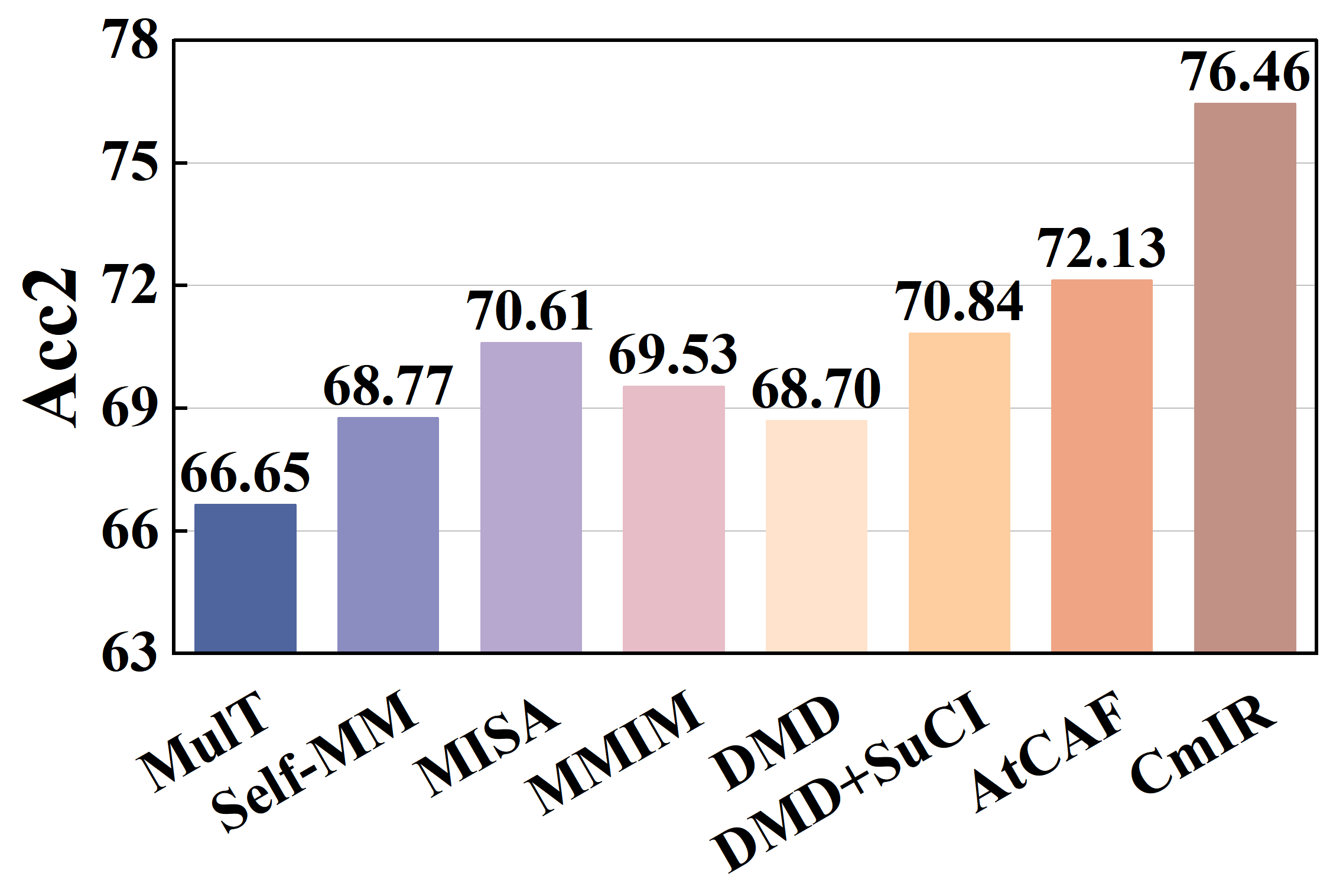}
        \vspace{-0.4cm}
        \caption{Results on MHD task}
    \end{subfigure}
    \begin{subfigure}[t]{0.48\linewidth}
        \centering
        \includegraphics[width=\linewidth]{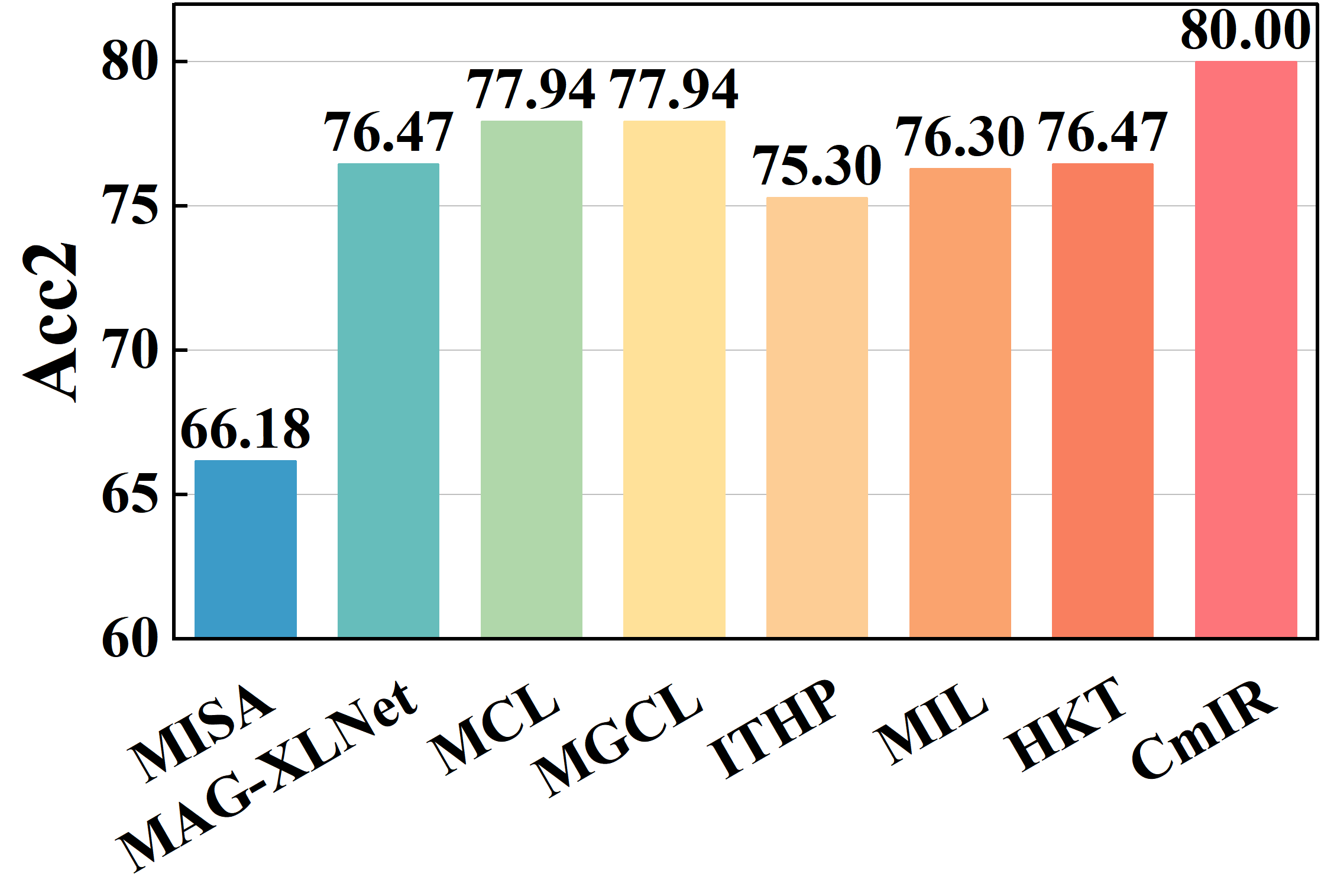}
        \vspace{-0.4cm}
        \caption{Results on MSD task}
    \end{subfigure}
    \vspace{-0.2cm}
    \caption{The results on (a) UR-FUNNY \cite{ur_funny} and (b) MUStARD \cite{msd} datasets.}
    \label{mhd_msd}
\end{figure}

\begin{table}[!t]
\centering
\caption{Results on CMU-MOSI (OOD). Following causal-based methods~\cite{CLUE2022,sun2023general}, {Acc2}$^\dagger$ and F1$^\dagger$ denote results considering neutral samples.}
\vspace*{-2mm}
\label{tab:OOD}
\resizebox{\linewidth}{!}{
\setlength\tabcolsep{5pt}
\renewcommand\arraystretch{1.1}
\begin{tabular}{c||c c c c c}
\noalign{\hrule height 1pt} 

\rowcolor{lightgray!45}
Methods
& Acc7 & {Acc2}$^\dagger$  & Acc2 & F1$^\dagger$ & F1  \\ 

\hline \hline
CLUE~\cite{CLUE2022}& 41.8 & 78.8 &79.9 & 78.8  &79.9 \\  
GEAR~\cite{sun2023general}  & - & \underline{80.5}  & \underline{82.1} & \underline{80.4}  & \underline{82.1}  \\
MulDeF~\cite{MulDeF2024}  & \underline{42.9} & 79.6   &81.5 & 79.7  &81.6  \\
Self-MM~\cite{mmsa} & 40.3 & 76.7 &78.1 & 76.7  &78.1 \\
ITHP~\cite{ithp} & 41.3 & 79.5 & 81.3 & 79.5 & 81.3 \\  
KAN-MCP~\cite{KAN-MCP2025} & 41.8 & 79.3 & 81.5 & 79.1 & 81.4 \\  
\hline
\rowcolor[HTML]{EBFAFF}
CmIR & \textbf{48.5} & \textbf{83.0} & \textbf{84.4} & \textbf{83.0} & \textbf{84.4} \\  

\noalign{\hrule height 1pt} 

\end{tabular}  
}
\vspace{-0.3cm}
\end{table}

\subsection{Results under OOD scenarios.}

The results under OOD scenarios is depicted in Table~\ref{tab:OOD}. It is observed that: I) All models degrade when moving from in-distribution to OOD settings, verifying that spurious correlations impede generalization; II) \textbf{CmIR delivers significantly stronger OOD performance} than standard multimodal baselines. Its advantage over ITHP~\cite{ithp} grows notably, with Acc2 improvement rising from 1.5 points to 3.5 points, and Acc7 improvement increasing from 2.1 points  to 7.2 points, underscoring the efficacy of our causal strategy; III) Compared to recent causality-based methods (CLUE, GEAR, MulDeF), CmIR consistently outperforms them in all metrics, highlighting its robustness in mitigating broad spurious correlations. This is because instead of focusing on specific bias or assumption, CmIR directly learns invariant representations that are stable across all environments, which is more general and applicable to various distribution shifts.

\begin{table*}[t]
\centering
\small
\caption{ \label{xxx}Discussion on noisy modalities on CMU-MOSI and CMU-MOSEI. MuBo denotes Multimodal Boosting.}
\vspace{-0.2cm}
\resizebox{1.9\columnwidth}{!}{%
\begin{tabular}{c|c|c|c|c|c|c|c|c|c}
\noalign{\hrule height 1pt} 
\rowcolor{lightgray!40}
 & &\multicolumn{4}{c|}{Gaussian Noise} & \multicolumn{4}{c}{OOD Noises (Laplace Noise and Random Erasing Noise)} \\
 \hline
 \rowcolor{lightgray!40}
 & 
 & \textbf{TMDC}  & \textbf{C-MIB} & \textbf{MuBo} 
& \cellcolor[HTML]{EBFAFF} \textbf{CmIR} & \textbf{TMDC}  & \textbf{C-MIB} 
& \textbf{MuBo} 
& \cellcolor[HTML]{EBFAFF} \textbf{CmIR}  \\
\cline{3-10}
\rowcolor{lightgray!40}
\multirow{-2}{*}{ } & \multirow{-2}{*}{\textbf{NR}} & Acc2/MAE & Acc2/MAE  & Acc2/MAE & \cellcolor[HTML]{EBFAFF} Acc2/MAE & Acc2/MAE  & Acc2/MAE & Acc2/MAE  & \cellcolor[HTML]{EBFAFF} Acc2/MAE  \\
\hline
\hline
\multirow{8}{*}{MOSI} 
& 0.1 &  87.4 / 0.748  &  87.8 / 0.670 &  86.7 / 0.678 & \cellcolor[HTML]{EBFAFF}   \textbf{88.1 / 0.615} &  87.2 / 0.769 &  \textbf{87.8} / 0.666 &  87.4 / 0.639 & \cellcolor[HTML]{EBFAFF} \textbf{87.8 / 0.638}  \\
& 0.2 &  86.6 / 0.741 &  \textbf{87.5} / 0.726 &  86.1 / 0.738 & \cellcolor[HTML]{EBFAFF} 87.4 / \textbf{0.621} &  86.7 / 0.861 &  87.5 / 0.689 &  87.3 / 0.681 & \cellcolor[HTML]{EBFAFF} \textbf{88.2 / 0.644}  \\
& 0.3 &  86.9 / 0.733 & 86.4 / 0.912  & 86.4  / 0.785 & \cellcolor[HTML]{EBFAFF}   \textbf{87.3 / 0.650} &  86.8 / 0.741 & 85.1 / 1.019  & \textbf{87.1} / 0.710 & \cellcolor[HTML]{EBFAFF} 86.9 / \textbf{0.678}  \\
& 0.4 &  85.5 / 0.792 &  83.2 / 1.366  & 85.5 / 0.841 & \cellcolor[HTML]{EBFAFF} \textbf{86.4 / 0.669} &  85.2 / 0.772 &  85.6 / 1.303 & 85.3 / 0.900 & \cellcolor[HTML]{EBFAFF} \textbf{85.8 / 0.663}  \\
& 0.5 &  85.0 / 0.912 &  84.9 / 1.660 &  86.1 / 1.172  & \cellcolor[HTML]{EBFAFF} \textbf{86.9 / 0.685} &  84.6 / 0.867 &  83.9 / 1.900 &  \textbf{86.9} / 1.114  & \cellcolor[HTML]{EBFAFF} 85.6 / \textbf{0.715}  \\
& 0.6 &  84.8 / 1.181 &  80.8 / 2.595 &  82.0 / 1.355  & \cellcolor[HTML]{EBFAFF} \textbf{87.1 / 0.680} &  83.1 / 0.884 &  86.5 / 2.272 &  85.0 / 1.060  & \cellcolor[HTML]{EBFAFF} \textbf{87.0 / 0.689}  \\
& 0.7 &  84.0 / 1.277 & 82.1 / 3.146 & 84.4 / 1.750 & \cellcolor[HTML]{EBFAFF} \textbf{86.1 / 0.740} &  82.0 / 0.966 & 84.0 / 3.516	 & 84.7 / 1.731 & \cellcolor[HTML]{EBFAFF} \textbf{85.6 / 0.771} \\
& \textbf{Avg} &  85.7 / 0.912 & 84.7 / 1.582 & 85.3 / 1.046  & \cellcolor[HTML]{EBFAFF} \textbf{87.0 / 0.666} &  85.1 / 0.837 & 85.8 / 1.624 & 86.2 / 0.976  & \cellcolor[HTML]{EBFAFF} \textbf{86.7 / 0.685}  \\
\hline
\multirow{8}{*}{MOSEI} 
& 0.1 &  86.6 / 0.618 & 86.1 / 0.545 &  86.4 / 0.544  & \cellcolor[HTML]{EBFAFF} \textbf{87.5 / 0.528} &  86.4 / 0.613 & 86.9 / 0.564 & 86.1 / 0.561  & \cellcolor[HTML]{EBFAFF} \textbf{87.1 / 0.523}  \\
& 0.2 &  86.2 / 0.593 & 84.5 / 0.582 &  86.6 / 0.557  & \cellcolor[HTML]{EBFAFF} \textbf{87.2 / 0.522} &  85.5 / 0.631 & 85.1 / 0.594 &  85.2 / 0.585  & \cellcolor[HTML]{EBFAFF} \textbf{86.6 / 0.523} \\
& 0.3 &  85.3 / 0.603 & 85.6 / 0.622 &  85.5 / 0.623  & \cellcolor[HTML]{EBFAFF} \textbf{86.8 / 0.531} &  85.0 / 0.667 & 85.9 / 0.665 &  85.4 / 0.610  & \cellcolor[HTML]{EBFAFF} \textbf{ 86.5 / 0.528 } \\
& 0.4 &  84.4 / 0.596 & 84.4 / 0.703  &  85.3 / 0.682 & \cellcolor[HTML]{EBFAFF} \textbf{86.6 / 0.535} &  84.9 / 0.682 & 84.8 / 0.767  &  84.4 / 0.714 & \cellcolor[HTML]{EBFAFF} \textbf{86.7 / 0.535} \\
& 0.5 &  84.6 / 0.592 &  83.7 / 0.875 &  84.1 / 0.724  & \cellcolor[HTML]{EBFAFF} \textbf{85.9 / 0.558} &  84.3 / 0.751 &  73.1 / 0.768 &  84.4 / 0.763  & \cellcolor[HTML]{EBFAFF} \textbf{86.3 / 0.563}  \\
& 0.6 &  83.8 / 0.602 &  82.4 / 1.054 &  85.4 / 0.924  & \cellcolor[HTML]{EBFAFF} \textbf{85.8 / 0.548} &  82.7 / 0.829 &  82.0 / 0.856 &  84.1 / 0.795  & \cellcolor[HTML]{EBFAFF} \textbf{86.3 / 0.644}  \\
& 0.7  &  83.4 / \textbf{0.623} & 80.5  / 1.404  & 80.3 / 1.125 & \cellcolor[HTML]{EBFAFF} \textbf{86.3} / 0.671 &  81.7 / 0.888 & \textbf{86.5} / 2.366  & 85.0 / 1.158 & \cellcolor[HTML]{EBFAFF} 85.4 / \textbf{0.663}  \\
& \textbf{Avg} &  84.9 / 0.604 & 83.9 / 0.826 &  84.8 / 0.740  & \cellcolor[HTML]{EBFAFF} \textbf{86.6 / 0.556} &  84.4 / 0.723 & 83.5 / 0.940 &  84.9 / 0.741  & \cellcolor[HTML]{EBFAFF} \textbf{86.4 / 0.568}  \\
\noalign{\hrule height 1pt} 
\end{tabular}}
 \end{table*}%

\subsection{Discussion on Noisy Modalities}\label{sec:noise}

(1) To assess the robustness of CmIR to modality noise, we corrupt all modalities of all training and testing samples with \textbf{Gaussian noise} (the noise rate NR is set at 10\% -70\%). 
The compared baselines include  TMDC \cite{zhuang2025tmdc}, C-MIB \cite{MIB} and Multimodal Boosting \cite{10224356}, which adopt the same training and testing settings as CmIR. Following prior work \cite{10224356}, we report Acc2 and MAE.
Table~\ref{xxx} shows that \textbf{CmIR outperforms competitive baselines across most metrics (particularly in MAE)}, and its \textbf{performance advantage becomes even more pronounced as the noise level increases}.
This is mainly because CmIR can more accurately identify and extract causal features from noisy inputs and maintain stable predictive ability for labels via the proposed constraints. 
These results indicate the robustness of CmIR in handling noise data.

(2) 
To assess the resilience of CmIR to \textbf{out-of-distribution (OOD) noises} not encountered in training, we adopt a mixed-noise evaluation strategy: training samples are contaminated with Gaussian noise, while testing samples are perturbed with distinct noise types (Laplace and random erasing). As presented in Table~\ref{xxx}, `CmIR (OOD)' obtains competitive results and significantly surpasses strong baselines. These findings confirm that \textbf{CmIR generalizes effectively to unseen noises}, highlighting its promising application potential.

\begin{figure*}
    \centering
    \begin{subfigure}[b]{0.24\linewidth} 
        \centering
        \includegraphics[width=\linewidth]{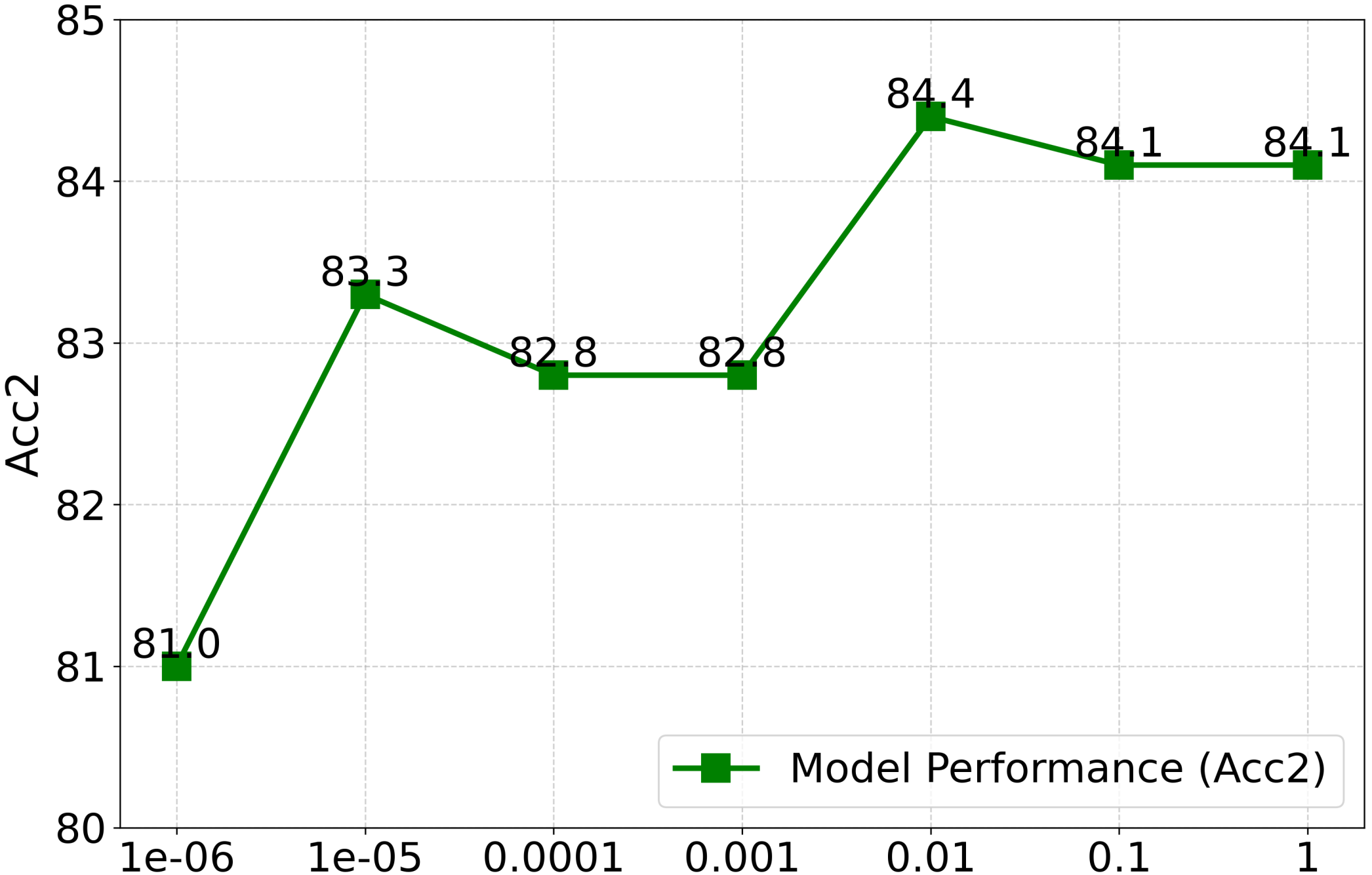}
        \caption{Weight $\lambda_1$}
    \end{subfigure}
    \begin{subfigure}[b]{0.24\linewidth}
        \centering
        \includegraphics[width=\linewidth]{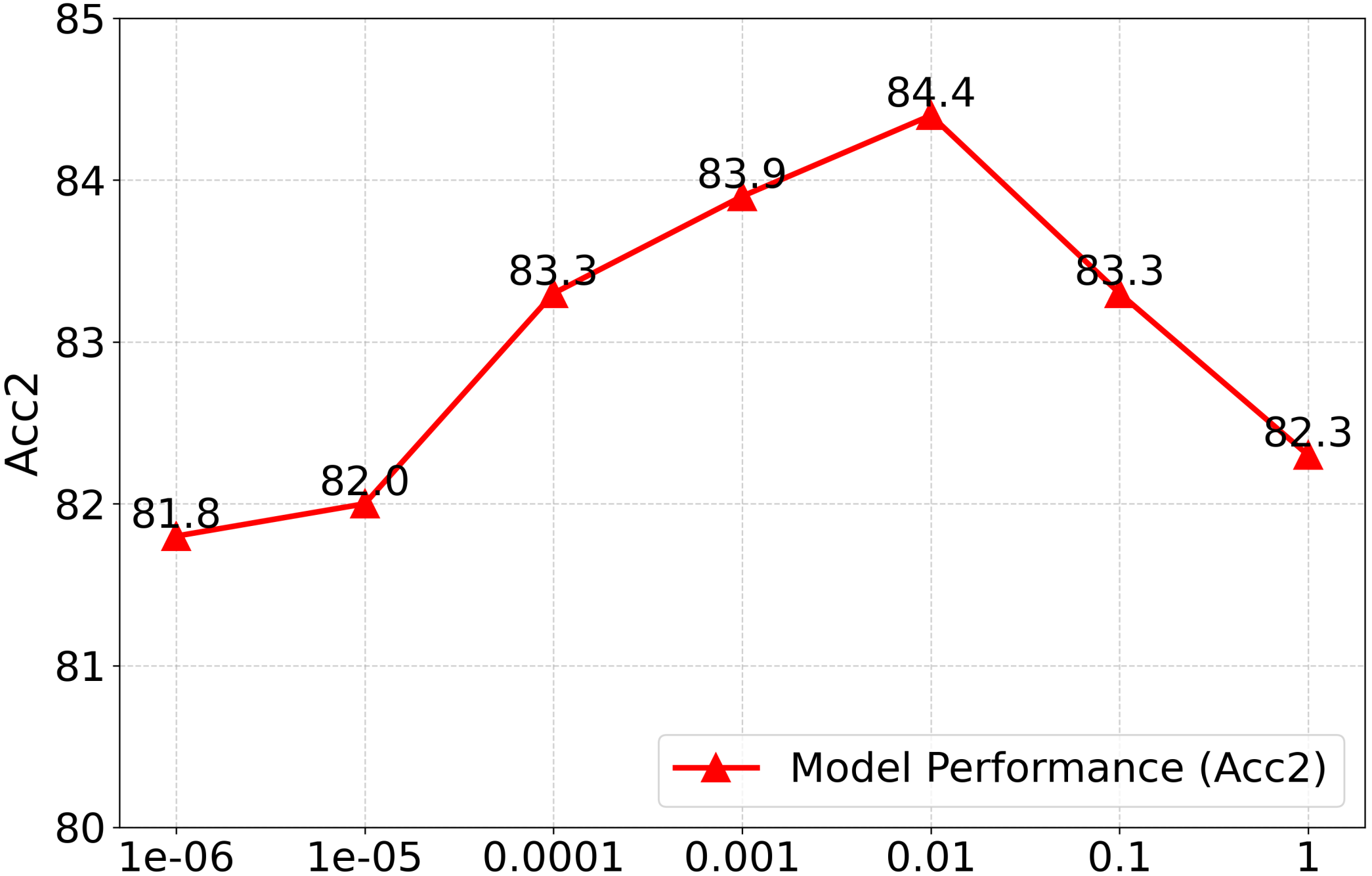}
        \caption{Weight $\lambda_2$}
    \end{subfigure}
    \begin{subfigure}[b]{0.24\linewidth}
        \centering
        \includegraphics[width=\linewidth]{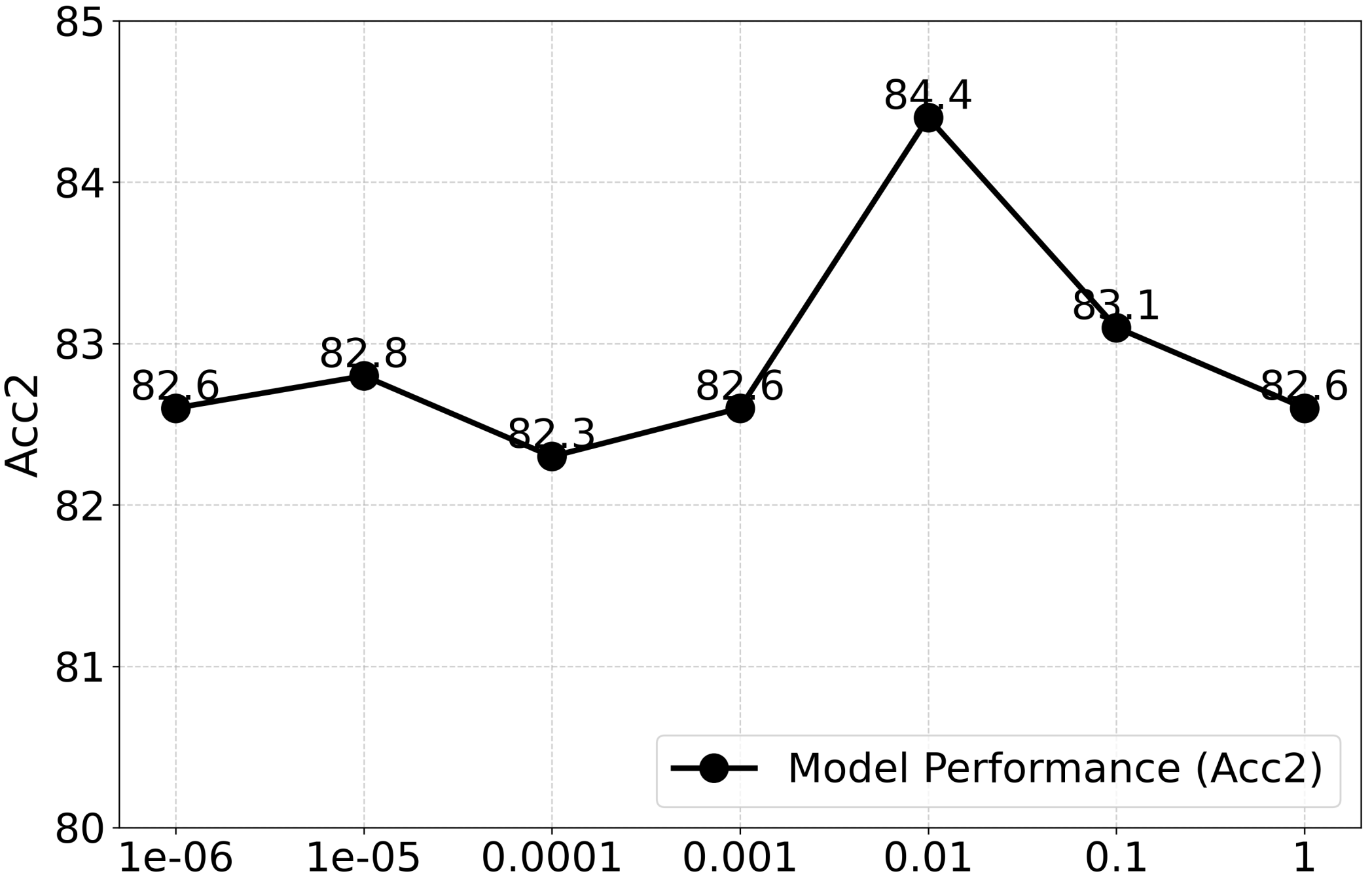}
        \caption{Weight $\lambda_3$}
    \end{subfigure}
    \begin{subfigure}[b]{0.24\linewidth}
        \centering
        \includegraphics[width=\linewidth]{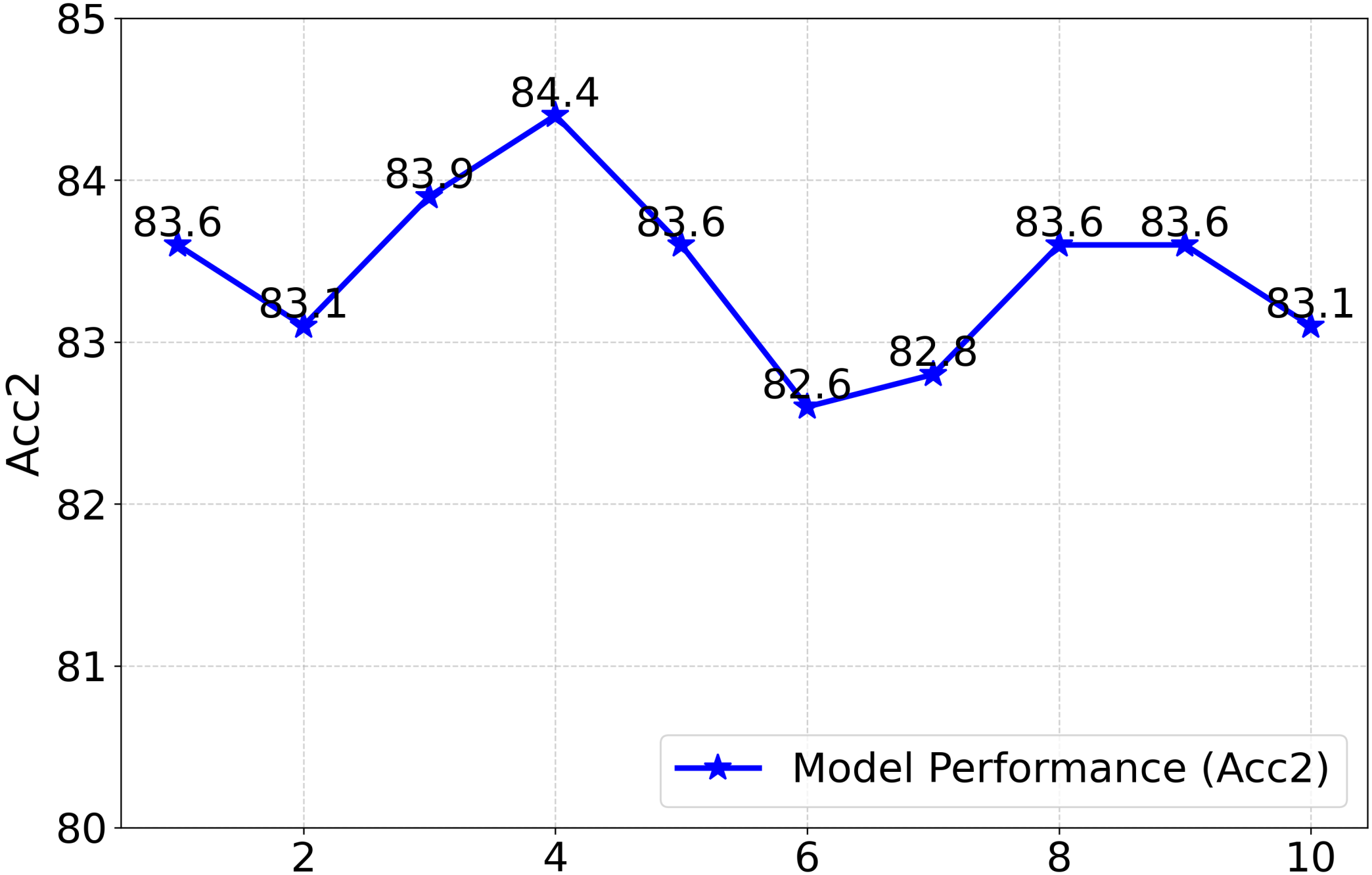}
        \caption{Environment Number $K$}
    \end{subfigure}
    \vspace{-0.2cm}
    \caption{Acc2 of CmIR w.r.t the change of constraint weights and the number of environments.}
    \label{9_mosi}
\end{figure*}

\subsection{Ablation Experiments} \label{sec:ablation}

\begin{table}
\vspace{-0.2cm}
\centering
\caption{\label{t3}Ablation experiments on CMU-MOSI.}
\vspace{-0.2cm}
\setlength\tabcolsep{5pt}
\renewcommand\arraystretch{1.2}
\resizebox{0.8\columnwidth}{!}{\begin{tabular}{c||ccc}
\noalign{\hrule height 1pt} 
\rowcolor{lightgray!40}

\textbf{Model}& \textbf{Acc7}$\uparrow$ & \textbf{Acc2}$\uparrow$ & \textbf{MAE}$\downarrow$ \\ 
\hline \hline
  Vanilla Framework   & 46.3 & 85.8  & 0.681 \\
  W/O $\mathcal{R}_{\text{inv}}^{(m)}$  & 45.7 & 86.7 & 0.652 \\
W/O $\mathcal{R}_{\text{dec}}^{(m)}$  & 45.7 & 86.9 & 0.671 \\
W/O $\mathcal{R}_{\text{rec}}^{(m)}$  & 47.7 & 88.1 & 0.623 \\
  \hline
\rowcolor[HTML]{EBFAFF}
CmIR  & \textbf{49.8} &   \textbf{89.6} & \textbf{0.616}  \\
\noalign{\hrule height 1pt} 
\end{tabular}}
\end{table}

\textbf{(1) Causal Inference}:
As presented in Table~\ref{t3}, in the case of `Vanilla Framework', we directly use the raw modality features for prediction. The model exhibits its sharpest performance decline (over 3 points in Acc2 and Acc7), suggesting the importance of learning causal invariant representations for more robust multimodal prediction and verifying our claim;
\textbf{(2) Invariance Constraint}:
As shown in `W/O $\mathcal{R}_{\text{inv}}^{(m)}$',
removing invariance constraint leads to a noticeable performance drop, 
because it is the core constraint to learn causal invariant representations. Without it, we cannot ensure that the learned $Z_{m}^{\text{inv}}$ is environment-invariant and the core idea of CmIR cannot be realized;
\textbf{(3) Mutual Information Constraint}:
 When $\mathcal{R}_{\text{dec}}^{(m)}$' is removed, the extent of performance decline is similar to that observed when $\mathcal{R}_{\text{inv}}^{(m)}$ is removed, indicating the necessity of minimizing the mutual information between $Z_{m}^{\text{inv}}$ and $Z_{m}^{\text{spu}}$, and demonstrating the effectiveness of our disentanglement framework. Compared with learning only invariant representations \cite{songlearning}, CmIR simultaneously learns both invariant and spurious representations while minimizing their mutual information, enabling the model to understand and learn the properties of invariant representations more easily and comprehensively;
\textbf{(4) Reconstruction Constraint}:
 When $\mathcal{R}_{\text{rec}}^{(m)}$  is removed, the performance also shows a noticeable decline, although the drop is the smallest among all ablations. This occurs because the reconstruction loss ensures the causal and spurious representations fully retain all information from the raw features, thereby preventing information loss and suboptimal disentanglement.

\subsection{\textbf{Hyperparameter Robustness Analysis}} \label{sec:hyper_a}

We evaluate the effectiveness of hyperparameters on CMU-MOSI (OOD), including the weights for invariance constraint $\lambda_1$, mutual information constraint $\lambda_2$, reconstruction constraint $\lambda_3$, and the number of environments $K$. 
As shown in Figure~\ref{9_mosi} (a), (b), and (c), 
when the values of weights are small, the performance of CmIR experiences a certain degree of degradation, as the effect of the constraints is not fully utilized. Among them, the performance drop is most pronounced when the weight of invariance constraint decreases, highlighting its importance. Conversely, when the values of $\lambda_2$ and $\lambda_3$ are too large, the performance also declines, likely because mutual information and reconstruction constraints dominate the learning process, preventing sufficient attention to the invariance constraint and the prediction loss. Moreover, as shown in Figure~\ref{9_mosi} (d), the performance remains relatively stable as the value of $K$ varies, indicating the robustness of CmIR. Overall, \textbf{when hyperparameters are varied across a wide range, CmIR’s performance consistently maintains a good level} (Acc2 $\geq$ 81\%), which to some extent demonstrates the stability of CmIR.

\begin{figure}
    \centering
    \begin{subfigure}[t]{0.5\linewidth} 
        \centering
        \includegraphics[width=\linewidth]{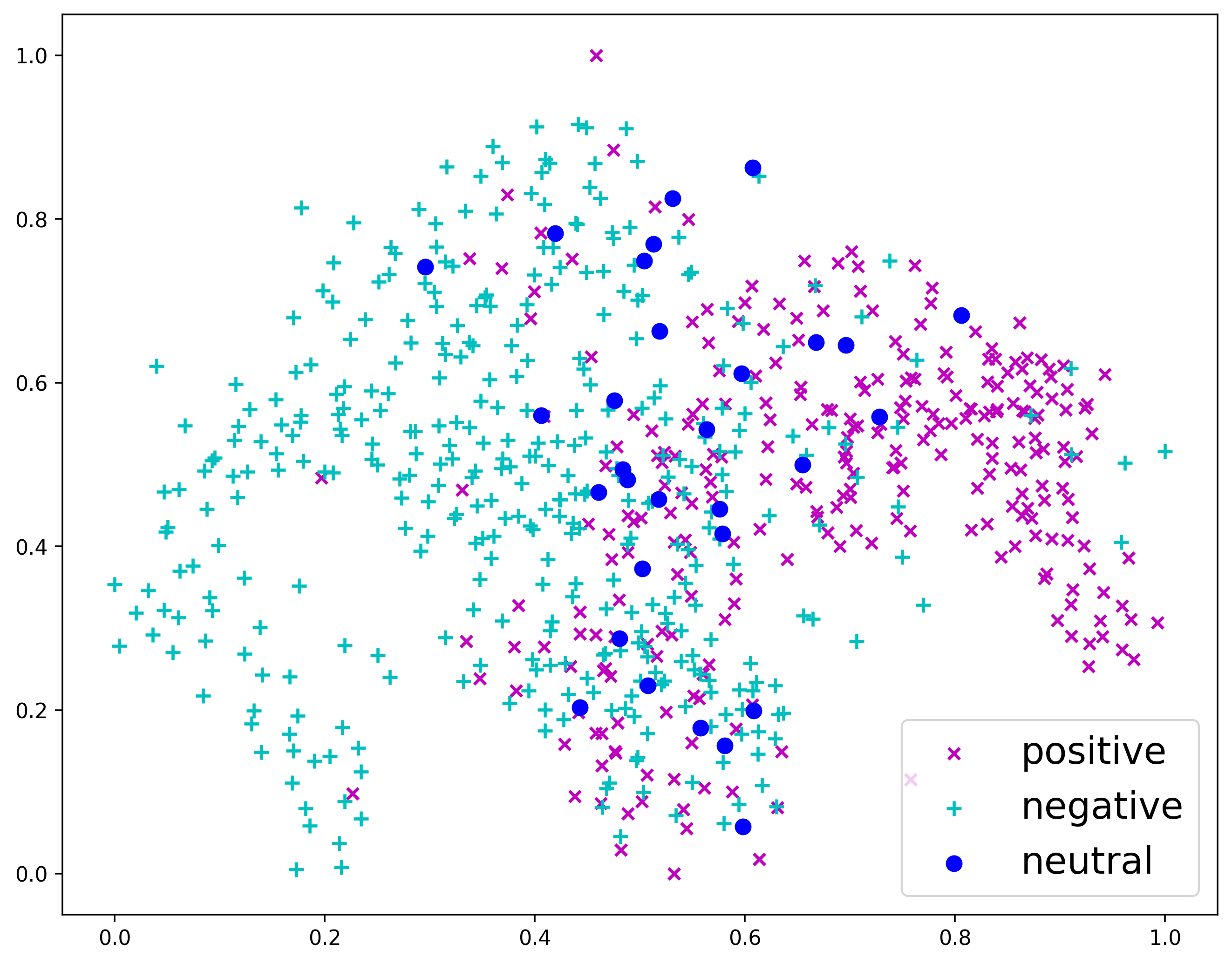}
        \vspace{-0.4cm}
        \caption{Regular Features}
    \end{subfigure}
    \begin{subfigure}[t]{0.48\linewidth}
        \centering
        \includegraphics[width=\linewidth]{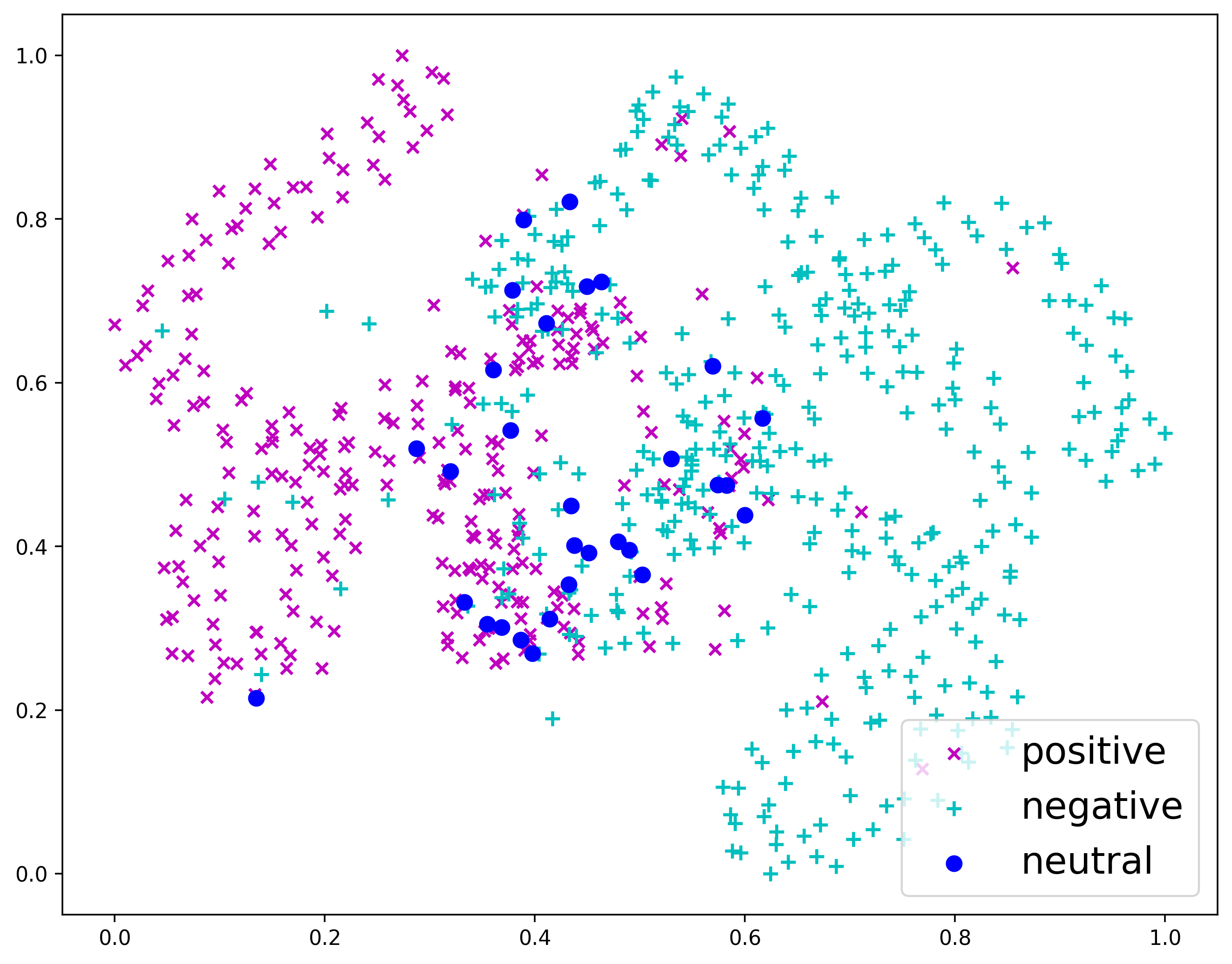}
        \vspace{-0.4cm}
        \caption{Invariant Features}
    \end{subfigure}
    \vspace{-0.2cm}
    \caption{\label{9}T-SNE visualization of language features without/with causal inference learning. 
    }
\end{figure}

\subsection{Visualization of Invariant Representations}

To demonstrate that CmIR indeed learns invariant representations that maintain stable predictive power for labels, we intervene with OOD noise on test samples and visualize the extracted causal language representations, alongside visualizing the language  representations obtained without causal inference training. As shown in Figure~\ref{9}, the causal representations learned by CmIR for different classes are well-separated in the feature space, with neutral samples concentrated between the positive and negative ones. In contrast, regular representations learned without CmIR exhibit substantial overlap, and neutral samples appear more scattered. This indicates that \textbf{even under noisy conditions, the causal representations learned by CmIR more accurately reflect label information}.

\section{Conclusion}
We propose CmIR for robust multimodal learning.
By disentangling each modality into invariant causal and spurious components, CmIR learns stable and causally-aware representations. We provide theoretical guarantees for CmIR, and demonstrate state-of-the-art results in multiple tasks. CmIR excels under distribution shifts and noisy-modality conditions, highlighting its practical robustness.

\section*{Limitations}

While CmIR demonstrates strong performance and robustness, our work has certain limitations. First, the environmental simulation via feature perturbation, while effective, may not fully capture the complexity of real-world distribution shifts. Future work could explore more sophisticated environment generation strategies or incorporate real-world multi-environment datasets. Second, the mutual information minimization constraint implemented via feature orthogonality is an approximation. More precise mutual information estimation techniques could be integrated for potentially better disentanglement, albeit at increased computational cost. These limitations, however, do not undermine the core theoretical contributions or the empirical effectiveness of the proposed framework.

\section{Acknowledgment}
This work is supported by the Guangdong Philosophy and Social Sciences Planning Project (No. GD26YJY34).

\bibliography{custom}

\appendix

\section{Detailed Theoretical Analysis}
\label{sec:theory_a}

Here we establish a detailed theoretical foundation for our causal approach to multimodal learning. We first establish the existence/definition and extractability of causal invariant modality representations, then prove their advantages in terms of generalization performance.

\subsection{Definition of Invariant Representations}
\label{subsec:invariant_representations_a}


\textbf{Theorem 1} (Definition of Causal Invariant Modality Representations)
\label{thm:invariant_representations_a}
Assume there exists a function class $\Phi_m = \{\phi_m: \mathcal{X}_m \rightarrow \mathcal{Z}_m^{\text{inv}}\}$ and a distribution distance measure $\mathcal{D}$ (e.g., KL divergence or Wasserstein distance) such that the following optimization problem has a solution:
\begin{equation*}
\begin{split}
\phi_m^*\! =\!  \arg\min_{\phi_m \in \Phi_m} \max_{e_1,e_2 \in \mathcal{E}} 
&\mathcal{D}(P(Y|\phi_m(X_m),E\!=\!e_1), \\
&P(Y|\phi_m(X_m),E=e_2))
\end{split}
\end{equation*}
Then $\phi_m^*(X_m)$ constitutes a causal invariant modality representation satisfying:
\begin{equation*}
\begin{split}
P(Y|\phi_m^*(X_m), E\!=\!e_1) =& P(Y|\phi_m^*(X_m), E\!=\!e_2), \\ 
& \forall e_1,e_2 \in \mathcal{E}
\end{split}
\end{equation*}

\begin{proof}
The proof follows from the invariance principle in causal inference~\cite{peters2016causal}. We proceed in three steps:

\textbf{Step 1 (Necessity of Invariance):} If a feature $Z$ contains causal information about $Y$, then the conditional distribution $P(Y|Z)$ should remain invariant across different environments~\cite{arjovsky2019invariant}. This is because causal mechanisms are stable under interventions on non-descendant variables in the causal graph.

\textbf{Step 2 (Sufficiency of the Optimization Objective):} The optimization objective directly minimizes the maximum discrepancy of $P(Y|\phi_m(X_m),E)$ across environments. By the properties of the distribution distance measure $\mathcal{D}$, $\mathcal{D}(P_1,P_2)=0$ if and only if $P_1=P_2$. Therefore, the optimal solution $\phi_m^*$ must satisfy:
\begin{equation*}
\begin{split}
\mathcal{D}(&P(Y|\phi_m^*(X_m),E=e_1), \\
&P(Y|\phi_m^*(X_m),E=e_2)) = 0
\end{split}
\end{equation*}
for all environment pairs $e_1,e_2 \in \mathcal{E}$, which implies:
\begin{equation*}
P(Y|\phi_m^*(X_m), E=e_1) = P(Y|\phi_m^*(X_m), E=e_2)
\end{equation*}

\textbf{Step 3 (Causal Interpretation):} The condition $P(Y|\phi_m^*(X_m), E=e_1) = P(Y|\phi_m^*(X_m), E=e_2)$ for all $e_1,e_2$ implies that $\phi_m^*(X_m)$ blocks all backdoor paths from $E$ to $Y$ that pass through modality $m$. By the backdoor criterion~\cite{peters2016causal}, this means $\phi_m^*(X_m)$ contains only causal features from $Z_m^{\text{inv}}$ and excludes spurious features from $Z_m^{\text{spu}}$, as the latter would create environment-dependent associations with $Y$.

This completes the proof that $\phi_m^*(X_m)$ is a valid causal invariant modality representation capturing only causal features.
\end{proof}


\subsection{Extractability of Invariant Representations}

\label{subsec:disentangled_representations_a}

While Theorem~\ref{thm:invariant_representations} establishes the definition of invariant representations, practical implementation requires extracting these representations from raw modalities while preserving all relevant information. This motivates our disentanglement approach.

\textbf{Theorem 2} (Theoretical Guarantee for Disentangled Representations)
\label{thm:disentangled_representations_a}
Consider encoder functions $g_m: \mathcal{X}_m \rightarrow (\mathcal{Z}_m^{\text{inv}}, \mathcal{Z}_m^{\text{spu}})$ and decoder functions $r_m: (\mathcal{Z}_m^{\text{inv}}, \mathcal{Z}_m^{\text{spu}}) \rightarrow \mathcal{X}_m$ that optimize the following objective:
\begin{align*}
&\min_{\{g_m,r_m,h\}_{m=1}^M} \mathbb{E}_{e\in\mathcal{E}}[\mathcal{L}_{\text{pred}}(Y, h(\{Z_m^{\text{inv}}\}_{m=1}^M))] \\
&+ \lambda_1 \sum_{m=1}^M \mathcal{R}_{\text{inv}}^{(m)} 
+ \lambda_2 \sum_{m=1}^M \mathcal{R}_{\text{dec}}^{(m)} 
+ \lambda_3 \sum_{m=1}^M \mathcal{R}_{\text{rec}}^{(m)}
\end{align*}
where:
\begin{itemize}
    \item $\mathcal{L}_{\text{pred}}$ is the prediction loss and $h$ is the prediction head 
    \item $\mathcal{R}_{\text{inv}}^{(m)} = \sum_{e_1 \in \mathcal{E}} \sum_{e_2 \in \mathcal{E}} 
    \mathcal{D}(P(Y|Z_m^{\text{inv}},E=e_1), 
    P(Y|Z_m^{\text{inv}},E=e_2))$ enforces invariance
    \item $\mathcal{R}_{\text{dec}}^{(m)} = I(Z_m^{\text{inv}}; Z_m^{\text{spu}})$ minimizes mutual information between invariant and spurious components
    \item $\mathcal{R}_{\text{rec}}^{(m)} = \|X_m - r_m(Z_m^{\text{inv}}, Z_m^{\text{spu}})\|^2$ ensures reconstruction capability
    \item $\lambda_1,\lambda_2,\lambda_3 > 0$ are hyperparameters
\end{itemize}
Assuming the label $Y$ is independent of environment $E$, the function classes $\{g_m,r_m, h\}$ have sufficient capacity and the data follows the SCM described in Section~\ref{subsec:invariant_representations}, then as $\lambda_1,\lambda_2,\lambda_3 \rightarrow \infty$, the optimal solution satisfies:
\begin{enumerate}
\item \(\displaystyle \lim_{\lambda_3\to\infty} \mathbb{E}\bigl[\|X_m - r_m(Z_m^{\mathrm{inv}},Z_m^{\mathrm{spu}})\|^2\bigr] = 0\) (perfect reconstruction is achieved)
    \item $Z_m^{\text{inv}} \perp\!\!\!\perp E$ (invariant component is environment-independent)
    \item $I(Y; Z_m^{\text{spu}} | Z_m^{\text{inv}}, E) = 0$ (spurious component contains no additional causal information)
\end{enumerate}

\begin{proof}
We prove the three claims in order. The limit \(\lambda_i\to\infty\) is understood in the sense of \textbf{tightening constraints}: as \(\lambda_i\) grows, the corresponding regularization term must vanish to keep the loss finite, provided the optimal loss remains bounded. We assume the feasible set (where all terms are finite) is non‑empty, which is reasonable given sufficient model capacity.

\textbf{Part 1 (Perfect Reconstruction):} The term $\lambda_3 \mathcal{R}_{\text{rec}}^{(m)}$ with $\lambda_3 \rightarrow \infty$ forces the reconstruction error to zero. By the properties of the squared $L2$ norm, \(\displaystyle \lim_{\lambda_3\to\infty} \mathbb{E}\bigl[\|X_m - r_m(Z_m^{\mathrm{inv}},Z_m^{\mathrm{spu}})\|^2\bigr] = 0\) if and only if $X_m = r_m(Z_m^{\text{inv}}, Z_m^{\text{spu}})$ almost surely. Thus, in the limit, the decoder can reconstruct the input with arbitrarily small error. Note that \textbf{exact} zero error may be unattainable for finite-dimensional representations, but the limiting statement suffices for theoretical analysis; in practice, taking \(\lambda_3\) sufficiently large yields negligible reconstruction error. This ensures the disentangled representations $(Z_m^{\text{inv}}, Z_m^{\text{spu}})$ preserve all information in the original modality $X_m$.

The reconstruction constraint $\mathcal{R}_{\text{rec}}^{(m)}$ is crucial for preventing degenerate solutions. It ensures that the disentangled representations form a sufficient statistic for $X_m$, preserving all information while separating causal from non-causal components. This is essential for maintaining performance in the source environments while improving generalization to new environments

\textbf{Part 2 (Environment Independence of Invariant Component):} The term $\lambda_1 \mathcal{R}_{\text{inv}}^{(m)}$ with $\lambda_1 \rightarrow \infty$ forces $\mathcal{D}(P(Y|Z_m^{\text{inv}},E=e_1), P(Y|Z_m^{\text{inv}},E=e_2)) = 0$ for all $e_1,e_2$. By Theorem~\ref{thm:invariant_representations}, this implies $P(Y|Z_m^{\text{inv}},E=e_1) = P(Y|Z_m^{\text{inv}},E=e_2)$ for all $e_1,e_2$.

Now, assume for contradiction that $Z_m^{\text{inv}}$ is not independent of $E$. Then there exist values $z^{\text{inv}}$, $e_1$, $e_2$ such that $P(Z_m^{\text{inv}}=z^{\text{inv}}|E=e_1) \neq P(Z_m^{\text{inv}}=z^{\text{inv}}|E=e_2)$. By the law of total probability:
\begin{align*}
P(Y|E=e_i) =& \int P(Y|Z_m^{\text{inv}}=z^{\text{inv}}) \\
&\times P(Z_m^{\text{inv}}=z^{\text{inv}}|E=e_i) dz^{\text{inv}}
\end{align*}
Since $P(Y|Z_m^{\text{inv}}) = P(Y|Z_m^{\text{inv}})$ but $P(Z_m^{\text{inv}}|E=e_1) \neq P(Z_m^{\text{inv}}|E=e_2)$, we must have $P(Y|E=e_1) \neq P(Y|E=e_2)$. However, in our SCM, $Y$ is causally independent of $E$ given the invariant features $\{X_m^{\text{inv}}\}$, and consequently given $\{Z_m^{\text{inv}}\}$. This contradiction implies $Z_m^{\text{inv}} \perp\!\!\!\perp E$.

\textbf{Part 3  (No additional causal information in \(Z_m^{\mathrm{spu}}\)):}
We prove \(I(Y; Z_m^{\mathrm{spu}} \mid Z_m^{\mathrm{inv}}, E)=0\) by contradiction, using the results from Part 1 and Part 2, the mutual information constraint, and the construction of virtual environments.

Assume, for contradiction, that  
\[
I := I(Y; Z_m^{\mathrm{spu}} \mid Z_m^{\mathrm{inv}}, E) > 0
\]

From Part~2, the invariance constraint gives \(P(Y\mid Z_m^{\mathrm{inv}}, E)=P(Y\mid Z_m^{\mathrm{inv}})\); hence  
\[
I = H(Y\mid Z_m^{\mathrm{inv}}) - H(Y\mid Z_m^{\mathrm{inv}}, Z_m^{\mathrm{spu}}, E) \tag{1}
\]
The inequality \(I>0\) implies  
\[
H(Y\mid Z_m^{\mathrm{inv}}, Z_m^{\mathrm{spu}}, E) < H(Y\mid Z_m^{\mathrm{inv}}). \tag{2}
\]
Thus, conditioning on \((Z_m^{\mathrm{spu}}, E)\) strictly reduces the entropy of \(Y\) compared to conditioning on \(Z_m^{\mathrm{inv}}\) alone. In particular, there exists a measurable set of positive measure on which the conditional distribution \(P(Y\mid Z_m^{\mathrm{inv}}, Z_m^{\mathrm{spu}})\) depends on \(Z_m^{\mathrm{spu}}\) in a non‑degenerate way.

By Part 1 (perfect reconstruction in the limit), the pair \((Z_m^{\mathrm{inv}}, Z_m^{\mathrm{spu}})\) determines \(X_m\) almost surely. Since \(Y\) is a function of the multimodal input, we can write:  
\[
Y = \Psi\bigl(Z_m^{\mathrm{inv}}, Z_m^{\mathrm{spu}}, \eta\bigr)
\]
where \(\eta\) is an independent noise term capturing irreducible uncertainty. The dependence on \(Z_m^{\mathrm{spu}}\) is essential because of (2).

Now consider two different environments \(e_1\) and \(e_2\), because the encoder \(g_m\) is deterministic and the same for all environments, the conditional distribution of \(Z_m^{\mathrm{spu}}\) given \(Z_m^{\mathrm{inv}}\) and \(E\) changes with \(e\). Formally, the map \(e \mapsto P(Z_m^{\mathrm{spu}}\mid Z_m^{\mathrm{inv}}, E=e)\) is injective; i.e., for \(e_1\neq e_2\) we have  
\[
P(Z_m^{\mathrm{spu}}\mid Z_m^{\mathrm{inv}}, E=e_1) \neq P(Z_m^{\mathrm{spu}}\mid Z_m^{\mathrm{inv}}, E=e_2)
\]
in the sense that the two conditional distributions are not equal almost everywhere.

Using the law of total probability, for any environment \(e\),
\[
\begin{split}
&P(Y\mid Z_m^{\mathrm{inv}}, E=e) = \\
&\int\! P(Y\!\mid\! Z_m^{\mathrm{inv}}, Z_m^{\mathrm{spu}}) \; dP(Z_m^{\mathrm{spu}}\!\mid\! Z_m^{\mathrm{inv}}, E=e) 
\end{split}
\tag{3}
\]
Because the environments can vary significantly, the family of mixing distributions \(\{P(Z_m^{\mathrm{spu}}\mid Z_m^{\mathrm{inv}},E=e)\}_{e\in\mathcal{E}}\) is rich enough to distinguish different integrands. In particular, if the integrand \(P(Y\mid Z_m^{\mathrm{inv}},Z_m^{\mathrm{spu}})\) is not constant in \(Z_m^{\mathrm{spu}}\) (which follows from \(I>0\)), then the value of the integral in (3) changes continuously with \(e\) when \(\alpha^{(e)}\) varies. Hence for \(e_1\neq e_2\), the two integrals cannot be equal. Consequently,
\[
P(Y\mid Z_m^{\mathrm{inv}}, E=e_1) \neq P(Y\mid Z_m^{\mathrm{inv}}, E=e_2)
\]
which contradicts Part~2 where we established that \(P(Y\mid Z_m^{\mathrm{inv}}, E=e)\) is constant across all environments. Therefore, our assumption \(I>0\) is false, and we must have
\[
I(Y; Z_m^{\mathrm{spu}} \mid Z_m^{\mathrm{inv}}, E) = 0
\]

Together, these three parts prove that the optimal solution satisfies all three claimed properties.
\end{proof}

\noindent\textbf{Remark on finite \(\lambda\).} The limit \(\lambda_i\to\infty\) is an idealization. In practice, taking \(\lambda_i\) sufficiently large (but finite) yields approximations where each regularization term is bounded by a small tolerance \(\epsilon\), and the conclusions hold up to \(\epsilon\) errors. This is standard in constrained optimization and does not affect the practical validity of the theorem.

\subsection{Distributionally Robust Risk Advantage of Invariant Representations}
\label{thm3_a}
Having established how to obtain causal invariant modality representations, we now prove their theoretical advantages for worst-case out-of-distribution risk under distribution shift.

\textbf{Theorem 3} (Distributionally Robust Risk Advantage of Invariant Representations)
\label{thm:ood_risk_advantage_a}
Let $\mathcal{H}$ be a hypothesis class over $\mathcal{X} \times \mathcal{Y}$, where $\mathcal{X} = \mathcal{X}_1 \times \mathcal{X}_2 \times \dots \times \mathcal{X}_M$ is the $M$-modal feature space and $\mathcal{Y}$ is the label space. Let $\mathcal{E}_{\text{all}}$ denote the set of all possible environments, each corresponding to a distribution $P^e(x, y)$. Let $h_{\text{inv}} \in \mathcal{H}$ be a predictor using invariant representations $Z^{inv} = \{Z_m^{\text{inv}}\}_{m=1}^M$, and $h_{\text{raw}} \in \mathcal{H}$ be a predictor using raw multimodal representations $X = \{X_m\}_{m=1}^M$. Assume:

1. \textbf{Invariance Condition}: The invariant representations satisfy $P^e(Y|Z^{\text{inv}}) = P^{e'}(Y|Z^{\text{inv}})$ for all $e,e' \in \mathcal{E}_{\text{all}}$.

2. \textbf{Information Sufficiency}: The mutual information between invariant representations and raw features satisfies $I(Z^{\text{inv}}; X) > c$ for some constant $c > 0$ ($Z^{\text{inv}}$ contains enough information related to $X$).

3. \textbf{Loss Function Regularity}: The loss function $\ell$ is $L$-Lipschitz continuous and bounded.

Then the worst-case out-of-distribution risk satisfies:
\begin{equation*}
R^{\text{OOD}}(h_{\text{inv}}) < R^{\text{OOD}}(h_{\text{raw}})
\end{equation*}
where $R^{\text{OOD}}(h) = \max_{e \in \mathcal{E}_{\text{test}}} R^e(h)$ and $R^e(h) = \mathbb{E}_{(x,y)\sim P^e}[\ell(h(x), y)]$.

\begin{proof}
We prove the theorem through four key steps, with additional explanations regarding the validity of assumptions and practical implications.

\textbf{Step 1 (Information Retention Property):} 

The information sufficiency assumption $I(Z^{\text{inv}}; X) > c$ is typically reasonable in practice because invariant features often capture fundamental aspects of data that remain stable across environments. For instance, in vision-language tasks, semantic content tends to remain consistent even when visual appearance changes. Formally, by the mutual information chain rule:
\begin{align*}
I(Y; X) &= I(Y; Z^{\text{inv}}, Z^{\text{spu}}) = I(Y; Z^{\text{inv}})\\
&+ I(Y; Z^{\text{spu}} | Z^{\text{inv}}) - I(Z^{\text{inv}}; Z^{\text{spu}} | Y)
\end{align*}

Since $I(Y; Z^{\text{spu}} | Z^{\text{inv}}) \leq H(Z^{\text{spu}} | Z^{\text{inv}})$ and $I(Z^{\text{inv}}; Z^{\text{spu}} | Y) \geq 0$, we have:
\begin{align*}
I(Y; X) &\leq I(Y; Z^{\text{inv}}) + H(Z^{\text{spu}} | Z^{\text{inv}}) \\
&= I(Y; Z^{\text{inv}}) + H(X | Z^{\text{inv}}) \\
&\leq I(Y; Z^{\text{inv}}) + H(X) - I(X; Z^{\text{inv}})
\end{align*}

Rearranging terms:
\begin{equation*}
I(Y; Z^{\text{inv}}) \geq I(Y; X) - (H(X) - I(X; Z^{\text{inv}}))
\end{equation*}

Setting $\epsilon(c) = H(X) - I(X; Z^{\text{inv}})$, when $I(X; Z^{\text{inv}}) > c$, we have $\epsilon(c) < H(X) - c$, ensuring that $I(Y; Z^{\text{inv}})$ approaches $I(Y; X)$ as $c$ increases.
This indicates that the invariant representations $Z^{\text{inv}}$ retain most of the predictive information about $Y$ that is present in the raw features $X$.

Moreover, regrading assumption 3, in real-world multimodal learning, the Lipschitz continuity assumption is widely applicable. For classification tasks using cross-entropy loss, when input features are normalized (as is common practice), the loss becomes Lipschitz continuous. Similarly, mean squared error loss for regression tasks is inherently Lipschitz continuous. This regularity ensures stable optimization and meaningful generalization bounds.
Since $\ell$ is assumed to be $L$-Lipschitz continuous, the risk gap is bounded \cite{songlearning}:
\begin{equation*}
R^e(h^*_{\text{inv}}) \leq R^e(h^*_{\text{raw}}) + L \cdot \epsilon(c)
\end{equation*}
where $h^*_{\text{inv}}$ and $h^*_{\text{raw}}$ are the optimal predictors using invariant and raw representations respectively. This shows that when $c$ is sufficiently large, $Z^{\text{inv}}$ almost completely preserves the information needed for prediction.


\textbf{Step 2 (Environment Invariance Property):} The invariance condition is realistic in many applications where causal factors remain stable despite environmental changes. For example, in object recognition, shape and category remain invariant while lighting, pose, and background may vary. This property ensures that $h_{\text{inv}}$ exhibits consistent performance across environments:
\begin{equation*}
R^e(h_{\text{inv}}) = R^{e'}(h_{\text{inv}}) = C, \quad \forall e,e' \in \mathcal{E}_{\text{all}}
\end{equation*}
where $C$ is a constant. Consequently:
\begin{equation*}
R^{\text{OOD}}(h_{\text{inv}}) = \max_{e \in \mathcal{E}_{\text{test}}} R^e(h_{\text{inv}}) = C
\end{equation*}
This consistency is a key advantage of invariant representations in distribution shift scenarios.

\textbf{Step 3 (Risk Variability of Raw Representations):} For predictor $h_{\text{raw}}$, the risk variability $\sigma = \max_{e,e' \in \mathcal{E}_{\text{test}}} |R^e(h_{\text{raw}}) - R^{e'}(h_{\text{raw}})|$ is typically large in real-world applications where distribution shifts are significant. This variability stems from the model's reliance on features that change with environment.

Let $R_{\max} = \max_{e \in \mathcal{E}_{\text{test}}} R^e(h_{\text{raw}})$ and $R_{\min} = \min_{e \in \mathcal{E}_{\text{test}}} R^e(h_{\text{raw}})$. The average risk $\bar{R} = \mathbb{E}_{e \in \mathcal{E}_{\text{test}}}[R^e(h_{\text{raw}})]$ satisfies:
\begin{align*}
R_{\max} - \bar{R} &= \frac{1}{|\mathcal{E}_{\text{test}}|} \sum_{e \in \mathcal{E}_{\text{test}}} (R_{\max} - R^e(h_{\text{raw}})) \\
&\geq \frac{R_{\max} - R_{\min}}{|\mathcal{E}_{\text{test}}|} = \frac{\sigma}{|\mathcal{E}_{\text{test}}|}
\end{align*}
which implies the existence of environment $e^*$ such that:
\begin{equation*}
R^{e^*}(h_{\text{raw}}) \geq \bar{R} + \frac{\sigma}{|\mathcal{E}_{\text{test}}|}
\end{equation*}
In practice, $\sigma$ is often substantial when dealing with significant domain shifts, making this inequality practically relevant.

\textbf{Step 4 (Worst-Case Risk Comparison):} This step reveals why invariant predictors typically outperform raw predictors in worst-case scenarios. The generalization errors $\delta_{\text{raw}}$ and $\delta_{\text{inv}}$ represent the gap between training and testing performance.

Notably, $\delta_{\text{inv}}$ is typically smaller than $\delta_{\text{raw}}$ because invariant features generalize better across environments. In contrast, $\sigma$ is often large in real-world scenarios with significant distribution shifts. For example, in cross-domain sentiment analysis, performance can vary dramatically between domains (e.g., movie reviews vs. product reviews), resulting in large $\sigma$ values.

From Steps 1 and 2:
\begin{equation*}
R^{\text{OOD}}(h_{\text{inv}}) = C \leq R^e(h^*_{\text{raw}}) + L \cdot \epsilon(c), \quad \forall e \in \mathcal{E}_{\text{test}}
\end{equation*}

By empirical risk minimization theory, standard generalization bounds apply:
\begin{align*}
R^e(h_{\text{inv}}) &\leq R^e(h^*_{\text{inv}}) + \delta_{\text{inv}}, \quad \forall e \in \mathcal{E}_{\text{test}}
\end{align*}
where $\delta_{\text{inv}}$ can be very small by realization using an expressive and suitable neutral network for the predictor, which decreases with the number of training samples and depend on model complexity.
Moreover, as $h^*_{\text{raw}}$ is the optimal predictor for raw features, we have:
\begin{align*}
R^e(h_{\text{raw}}) &\geq R^e(h^*_{\text{raw}}), \quad \forall e \in \mathcal{E}_{\text{test}} 
\end{align*}

From Step 3:
\begin{align*}
R^{\text{OOD}}(h_{\text{raw}}) &\geq R^{e^*}(h_{\text{raw}}) \\
&\geq \bar{R} + \frac{\sigma}{|\mathcal{E}_{\text{test}}|} \\
&\geq \mathbb{E}_{e \in \mathcal{E}_{\text{test}}}[R^e(h^*_{\text{raw}})] + \frac{\sigma}{|\mathcal{E}_{\text{test}}|}
\end{align*}

Similarly:
\begin{align*}
R^{\text{OOD}}(h_{\text{inv}}) &= C \\
&= \mathbb{E}_{e \in \mathcal{E}_{\text{test}}}[R^e(h_{\text{inv}})] \\
&\leq \mathbb{E}_{e \in \mathcal{E}_{\text{test}}}[R^e(h^*_{\text{raw}})] + L \cdot \epsilon(c) + \delta_{\text{inv}}
\end{align*}

Combining these results:
\begin{align*}
&R^{\text{OOD}}(h_{\text{raw}}) - R^{\text{OOD}}(h_{\text{inv}}) \\
&\geq \left(\mathbb{E}_{e \in \mathcal{E}_{\text{test}}}[R^e(h^*_{\text{raw}})] + \frac{\sigma}{|\mathcal{E}_{\text{test}}|}\right) \\
&\quad - \left(\mathbb{E}_{e \in \mathcal{E}_{\text{test}}}[R^e(h^*_{\text{raw}})] + L \cdot \epsilon(c) + \delta_{\text{inv}}\right) \\
&= - L \cdot \epsilon(c) - \delta_{\text{inv}} + \frac{\sigma}{|\mathcal{E}_{\text{test}}|}
\end{align*}

The inequality $R^{\text{OOD}}(h_{\text{raw}}) > R^{\text{OOD}}(h_{\text{inv}})$ holds when:
\begin{equation*}
\sigma > |\mathcal{E}_{\text{test}}|( L \cdot \epsilon(c) + \delta_{\text{inv}})
\end{equation*}

This condition is typically satisfied in practice because:
\begin{itemize}
    \item $\sigma$ tends to be large in real-world domain shifts, as environments often differ significantly in their distributional properties.
    \item $\delta_{\text{inv}}$ is small because invariant features generalize well across environments and we can adopt proper realization of the predictor $h$ to reduce $\delta_{\text{inv}}$.
    \item $\epsilon(c)$ can be made small by ensuring $Z^{\text{inv}}$ captures sufficient information from $X$.
\end{itemize}


This proves that under realistic conditions, predictors based on invariant representations achieve strictly lower worst-case out-of-distribution risk than those using raw representations.
\end{proof}

\section{Unimodal Networks} \label{unimodal_a}


This section outlines the architecture of our unimodal networks and elaborates on the steps for generating unimodal representations, which serve as the foundation for subsequent causal analysis. To make a fair comparison, following established practices in recent work \cite{ithp,mgcl}, we utilize pre-trained language models \cite{deberta,lan2019albert} to derive high-quality textual features. The following steps outline the language network's workflow for all downstream tasks
\begin{equation}
\setlength{\abovedisplayskip}{3pt}
\setlength{\belowdisplayskip}{3pt}
\label{eq6}
\begin{split}
   \bm{\hat{X}}_{l}&=\text{PLM}(\bm{U}_l;\ \theta_{l}) \in \mathbb{R}^{T_l \times d_l }\\
     \bm{X}_{l}  &=  ( \bm{\hat{X}}_{l} \bm{W}_{pro} + \bm{b}_{pro} ) \in \mathbb{R}^{T_l \times d }\\
\end{split}
\end{equation}
where PLM indicates the pre-trained language model, $\bm{U}_l $ is the input token sequence and $T_l$ represents the sequence length. $\bm{W}_{pro} \in \mathbb{R}^{d_l\times d}$ and $\bm{b}_{pro}\in \mathbb{R}^{1\times d}$ are trainable parameters that map the output dimensionality of the language network to the shared feature dimensionality $d$.  
In MSA, the acoustic and visual networks employ transformer encoders \cite{transformer} and operate according to the following steps ($m\in \{a,v\}$):
\begin{equation}
\setlength{\abovedisplayskip}{3pt}
\setlength{\belowdisplayskip}{3pt}
\label{eq7}
\begin{split}
&\bm{\hat{X}}_{m}=\operatorname{Conv} 1 \mathrm{D} ( \bm{U}_{m};\ K_m ) \in \mathbb{R}^{T_m \times d}\\
  & \bm{X}_{m} = \text{Transformer}(\bm{\hat{X}}_{m};\ \theta_m) \in \mathbb{R}^{ T_m \times d}
\end{split}
\end{equation} 
where $\bm{U}_{m} \in \mathbb{R}^{T_m \times d_m}$ is the extracted raw feature sequence (see Section~\ref{sec:fea_detail} for the extraction derails), $\operatorname{Conv} 1 \mathrm{D}$ indicates the temporal convolution whose kernel size $K_m$ is set to 3. Nota that for the CH-SIMS-v2 dataset, we use the same feature set as in previous works \cite{kuda,simsv2}, which are feature vectors instead of sequences. Therefore, we simply use multi-layer perception networks as the unimodal networks for visual and acoustic modalities.

In the MHD and MSD tasks, to more effectively model humor-specific cues, we follow prior work \cite{HKT} by extracting an additional Humor-Centric Feature (HCF) from the language modality. This HCF serves as a fourth modality and is represented as $\bm{U_h} \in \mathbb{R}^{T_l \times d_h}$ (detailed in \cite{HKT}). In addition, each data sample in the MHD and MSD tasks includes both a target punchline segment and its preceding context. We merge the feature sequences of the punchline and context along the temporal axis to construct the unimodal input representations
$\bm{U_m} \in \mathbb{R}^{ T_m \times d_m} \  (m\in \mathcal{M}= \{a, v, l, h\})$.
The unimodal network for the HCF modality is analogous to the framework employed for the visual and acoustic modalities. The specific steps of the transformer-based unimodal networks for MHD and MSD are delineated as follows ($m\in \{a,v,h\}$):
\begin{equation}
\setlength{\abovedisplayskip}{3pt}
\setlength{\belowdisplayskip}{3pt}
\label{eq8}
\begin{split}
 \bm{\hat{X}}_{m}\!&=\!\text{Transformer}(\bm{U}_m ;\ \theta_m )\in \mathbb{R}^{ T_m \times d_m} \\
    \bm{X}_{m} &=  \operatorname{Conv} 1 \mathrm{D}( \bm{\hat{X}}_{m};\ K_m) \in \mathbb{R}^{ T_m \times d}\\
\end{split}
\end{equation}
Finally, we employ a straightforward linear layer to fuse the language and HCF modalities, thereby reducing complexity for the following model stages:
\begin{equation}
\setlength{\abovedisplayskip}{3pt}
\setlength{\belowdisplayskip}{3pt}
\label{eq8_2}
\begin{split}
 \bm{X}_l\longleftarrow\text{Linear}(\bm{X}_l \oplus \bm{X}_h &;\ \theta_{lin} )\in \mathbb{R}^{ T_l \times d}
\end{split}
\end{equation}

For all unimodal representations $\bm{X}_m\in \mathbb{R}^{ T_m \times d}$, we perform mean pooling at the time dimension to obtain the final unimodal representations $X_m\in \mathbb{R}^{ 1 \times d}$.

\section{Datasets}



(1) \textbf{CMU-MOSI} \cite{zadeh2016multimodal}:
This dataset is a standard benchmark for MSA, encompassing more than 2,000 online video clips collected from the Internet. Each clip is labeled with a sentiment score on a -3 to 3 Likert scale, where 3 and -3 denote extreme positive and negative sentiments, respectively.

(2) \textbf{The OOD version of CMU-MOSI}~\cite{CLUE2022}: CMU-MOSI (OOD) is built using a modified simulated annealing algorithm~\cite{aarts1987simulated}, which iteratively adjusts the test distribution. The resulting significant shifts in word–sentiment correlations relative to the training set establish it as a challenging benchmark for evaluating model robustness to distribution shifts in MSA.


(3) \textbf{CMU-MOSEI} \cite{MOSEI}: 
The CMU-MOSEI dataset is a large-scale, widely-adopted benchmark for MSA, collected from online videos. Its key characteristics include: (I) Scale: over 22,000 video clips; (II) Source: more than 1,000 YouTube speakers and 250+ topics, randomly sampled; (III) Annotation: each clip has two labels—a six-class emotion category and a sentiment score ranging from -3 (strongly negative) to 3 (strongly positive). For the MSA task, our evaluation adopts the sentiment labels of CMU-MOSEI, which are consistent with the scale used in CMU-MOSI.

(4) \textbf{CH-SIMS-v2} \cite{simsv2}: 
CH-SIMS-v2 serves as a Chinese MSA benchmark with the following characteristics: (I) Source: Videos collected from 11 scenarios (interviews, talk shows, films, etc.) to mimic real-world interaction; (II) Quality: Filtered to retain high-quality acoustic and visual streams; (III) Split: Partitioned into training, validation, and test sets in a 9:2:3 ratio, corresponding to 2,722, 647, and 1,034 segments respectively; (IV) Label Distribution: The training set contains 921 negative, 433 weakly negative, 232 neutral, 318 weakly positive, and 818 positive samples.


(5) \textbf{UR-FUNNY} \cite{ur_funny}: 
Derived from TED talk videos involving 1,741 speakers, UR-FUNNY serves as a benchmark for multimodal humor detection (MHD). Each data sample includes a multimodal punchline segment and its preceding context segments, the latter being provided to support contextual modeling. The dataset is built by identifying punchlines via the laughter tag in transcripts. Video segments followed by laughter are treated as positive samples, while those without laughter form negative samples. It is partitioned into 7,614 training, 980 validation, and 994 test instances.


(6) \textbf{MUStARD} \cite{msd}: 
The MUStARD dataset is designed for multimodal sarcasm detection (MSD), comprising video segments sourced from popular TV series including Friends, The Big Bang Theory, The Golden Girls, and Sarcasmaholics. The collection contains 690 human-annotated segments, labeled as either sarcastic or non-sarcastic. Similar to UR-FUNNY, it provides contextual clues by incorporating both the target punchline and the preceding dialogue segments for each sample.

\section{Evaluation Metrics}


For CMU-MOSI and CMU-MOSEI datasets, we adopt the following evaluation metrics: (1) \textbf{Acc7}: the accuracy of classifying sentiment scores into seven discrete classes; (2) \textbf{Acc2}: the binary accuracy for differentiating between positive and negative sentiments; (3) \textbf{F1 score}: a harmonic mean that balances precision and recall for binary sentiment classification; (4) \textbf{MAE}: the mean absolute error between model predictions and sentiment labels; and (5) \textbf{Corr}: the correlation coefficient reflecting the strength and direction of the relationship between  predictions and sentiment labels. For Acc7, predictions are rounded to the nearest integer within the scale from -3 to 3. When calculating Acc2 and F1 score, neutral segments are not considered. And the neutral segments are included in the calculations of MAE, Corr, and Acc7. For the CH-SIMS-v2 dataset, we use Acc5, Acc3, Acc2, F1 score, MAE, and Corr as in previous works \cite{simsv2,kuda}. For the MHD and MSD tasks, we report the binary accuracy (i.e., humorous or non-humorous, sarcastic or non-sarcastic) of the model. 
 
\section{Feature Extraction Details}\label{sec:fea_detail}

Regarding the \textbf{visual modality}, following previous approaches \cite{hycon,ithp}, 
Facet \footnote{iMotions 2017. https://imotions.com/} is used to gather an array of visual attributes such as facial action units and facial landmarks for the MSA task. Facial feature extraction for the CH-SIMS-v2 dataset aligns with established practice \cite{kuda,simsv2}, utilizing OpenFace \cite{baltrusaitis2018openface} to obtain measures such as 68 facial landmarks, 17 action units, head pose, head orientation, and eye gaze direction.
For MHD and MSD, in line with prior approaches \cite{HKT,mcl}, OpenFace 2 \cite{baltrusaitis2018openface} is utilized for the extraction of facial action unit features as well as rigid and non-rigid facial shape parameters.
For \textbf{acoustic modality}, COVAREP \cite{Degottex2014COVAREP} is used for the extraction of a sequence of acoustic features, including 12 Mel-frequency cepstral coefficients, pitch tracking, speech polarity, etc. For the CH-SIMS-v2 dataset, acoustic features are represented as 25-dimensional eGeMAPS low-level descriptors (LLD), extracted via OpenSmile \cite{eyben2010opensmile} at 16 kHz.
For \textbf{language modality}, following state-of-the-art methods \cite{ithp}, DeBERTa \cite{deberta} and BERT \cite{bert} are employed to learn informative language representations. For the CH-SIMS-v2 dataset, we follow prior works \cite{kuda,simsv2} and utilize BERT \cite{bert} to obtain textual features.
For MHD and MSD, following prior methods \cite{mgcl,mcl}, ALBERT \cite{lan2019albert} is adopted. 

The input feature dimensionality for each modality is summarized in Table~\ref{tfeat}. 

\begin{table}
\centering
 \caption{ \label{tfeat}The unimodal feature dimensionality of different datasets.
 }
 \vspace{-0.cm}
\resizebox{1.\columnwidth}{!}{\begin{tabular}{c|c|c|c|c}
 \noalign{\hrule height 1pt} 
 \rowcolor{lightgray!40}
    & Language & Acoustic & Visual & HCF \\
 \hline
   CMU-MOSI & 768 & 74 &  47 & -  \\
   CMU-MOSI (OOD) & 768 & 74 &  47 & -  \\
   CMU-MOSEI & 768  & 74 & 35 & - \\
   CH-SIMS-v2 & 768 & 25 & 177 & - \\
   UR-FUNNY & 768 & 60 & 36 & 4  \\
   MUStARD & 768 & 60 & 36 & 4  \\
  \noalign{\hrule height 1pt} 
 \end{tabular}}
\end{table}%

\begin{figure}
\centering
\includegraphics[width=1.0\linewidth]{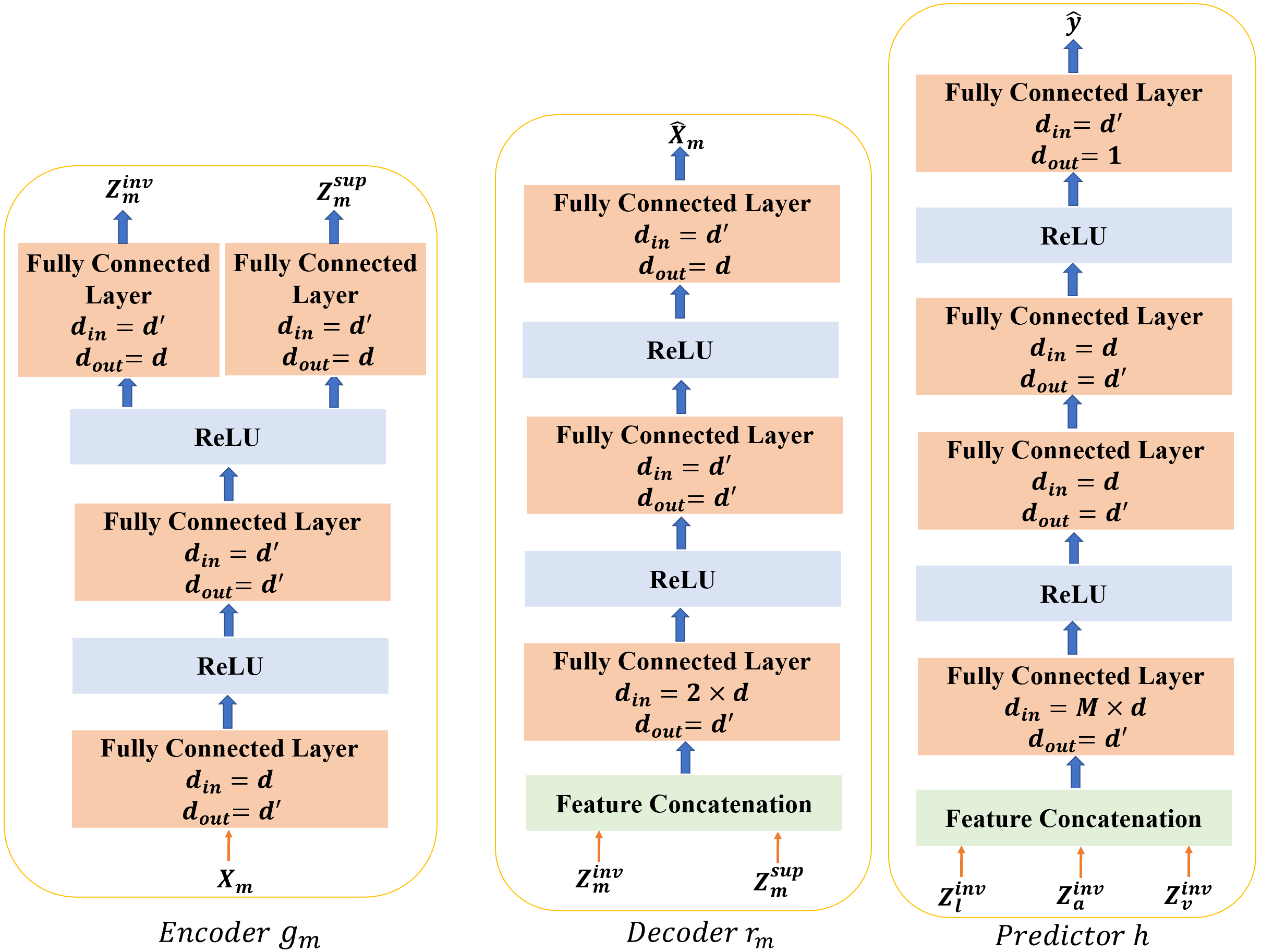}
\caption{\label{structure}The structures of encoder, decoder, and  predictor. $d^{'}$ represents the hidden dimensionality. } 
\end{figure}

\section{Experimental Details}\label{sec:exper_detail}

(1) \textbf{Hyperparameter Setting}: 
Our proposed CmIR is developed with PyTorch 1.13.1 on an NVIDIA RTX3090 GPU (CUDA 11.6). Training utilizes the AdamW optimizer \cite{loshchilov2017decoupled}. In line with prior work \cite{hycon}, optimal hyperparameters are determined through an extensive random grid search of 50 iterations on the validation set. Using the identified best configuration, the model is retrained five times, and the final performance is averaged across these runs. The specific hyperparameter settings can be found in Table~\ref{t2223}. Note that the modality-specific covariance matrix $\Sigma_m$ in Eq.~\ref{eq_inva} is set as the identity matrix.

The structures of the predictive attention weight generator and the predictor are shown in Figure~\ref{structure}. 

\begin{table*}[t]
\centering
 \caption{ \label{t2223}Hyperparameter Settings of CmIR. MAE, MSE and BCE denote mean absolute error, mean square error and binary cross-entropy, respectively.}
 \resizebox{2.0\columnwidth}{!}{\begin{tabular}{c|c|c|c|c|c|c}
  \noalign{\hrule height 1pt} 
\rowcolor{lightgray!40}
     & CMU-MOSI & CMU-MOSEI & CH-SIMS-v2 & MUStARD & UR-FUNNY & CMU-MOSI (OOD) \\
 \hline
 Loss Function  &  MSE & MSE & MAE &  BCE & BCE & MSE \\
  Batch Size  &  48 & 50 &  50 & 32 & 64 & 48 \\
  Learning Rate  & 1e-5  & 1e-5 & 5e-4 &  2e-5 & 7e-6 & 1e-5 \\
 Number of Environments $K$ & 1 & 5 & 2 &  4 & 2 & 4 \\
 Noise Intensity $\alpha^{(1)}$ & 0.1 & 1 & 0.1 &  0.1 & 0.1 & 0.1 \\
   Weight $\lambda_1$ & 0.1 & 0.1 & 0.01 &  0.05 & 0.001 & 0.01 \\
   Weight $\lambda_2$ & 0.001 & 0.1 & 0.01 & 0.01 & 0.01 & 0.01 \\
   Weight $\lambda_3$ & 0.05 & 0.01 & 0.1 & 0.05 & 0.01 & 0.01 \\
    Shared Dimensionality $d$ &  150 & 150 & 100 & 128 & 120 & 150 \\
   Hidden Dimensionality $d^{'}$ & 256 & 128 & 100 & 48 & 120 & 150 \\
   Weight $\alpha$ in Mutual Information Constraint &  1 & 1 & 0.01 & 0.1 & 0.1 & 1 \\
  \noalign{\hrule height 1pt} 
 \end{tabular}}
\end{table*}%

(2) \textbf{Protocol for the evaluation of noisy modalities}: To evaluate model robustness with noisy modalities, we produce corrupted features by employing the equation below:
\begin{equation}
\setlength{\abovedisplayskip}{2pt}
\setlength{\belowdisplayskip}{2pt}
\label{con1_4}
 X^{n}_{m} = (1 - NR)  \cdot X_{m} + NR \cdot \mathcal{\bm{N}}
\end{equation}
where $X_{m}$ is the unimodal representation, $NR$ is the noisy rate ranging from 0.1 to 0.7, $\mathcal{\bm{N}}$ is the Gaussian noise data of mean \textbf{0} and variance \textbf{1}, and $X^{n}_{m}$ is the noisy unimodal representation that is used to learn invariant representation. To simulate realistic noise, the described noise mixing process is applied to every modality across all samples. Perturbations are introduced at the feature level because input noise ultimately propagates to feature representations. This approach aligns with common practice in MAC, where a standardized feature set is typically used to ensure fairness and protect privacy, making feature-level noise injection both reasonable and practical. For a fair comparison, all baselines \cite{niat,MIB,10224356} are reproduced using the same training and testing protocols as our CmIR.

To further assess the robustness of CmIR under out-of-distribution (OOD) noise conditions (Table~\ref{xxx}), we evaluate its performance against two additional corruption types: Laplace noise and random erasing noise (the latter simulates data loss by randomly zeroing out a subset of features). During this evaluation, each sample has an equal 50\% chance of being corrupted either by Laplacian noise via Eq.~\ref{con1_4} or by random feature dropout (missing), where the dropout rate is controlled by the noise ratio $NR$.

\section{Baselines}

The causality-based baselines include:

(1)  \textbf{Subject Causal Intervention 
} (\textbf{SuCI}) \cite{xu2025debiased}: It introduces a simple yet effective causal intervention module designed to decouple the influence of subjects as unobserved confounders, thus obtaining unbiased predictions through true causal effects;
(2)  \textbf{CounterfactuaL mUltimodal sEntiment}  (\textbf{CLUE})~\citep{CLUE2022}:  It leverages causal inference and counterfactual reasoning to prune away spurious direct textual influences, preserving only the genuine indirect multimodal effects, thereby strengthening generalization to out-of-distribution data;
(3) \textbf{General dEbiAsing fRamework} (\textbf{GEAR})~\citep{sun2023general}: To enhance out-of-distribution robustness, it distinguishes robust features from biased ones, quantifies sample-level bias, and employs inverse probability weighting to de-emphasize highly biased samples;
(4) \textbf{Multimodal Debiasing Framework} (\textbf{MulDeF})~\citep{MulDeF2024}: It integrates causal intervention with front-door adjustment and multimodal causal attention during the training phase. At inference, it applies counterfactual reasoning to mitigate both verbal and nonverbal biases, which enhances out-of-distribution generalization.;
(5) \textbf{Attention-based Causality-Aware FusionAttention-based Causality-Aware Fusion} (\textbf{AtCAF})~\citep{huang2025atcaf}:  It learns causality-aware multimodal representations for sentiment analysis through a dedicated text debiasing module and counterfactual attention across modalities.

The baselines for handling noisy modality include:

(1) \textbf{Multimodal Boosting} \cite{10224356}: It is built upon multiple base learners organized in a boosting-like manner, with each learner addressing different facets of the multimodal data. It is further equipped with a contribution learning module that dynamically estimates the contribution and noise degree of individual learners;
(2) \textbf{Complete Multimodal Information Bottleneck} (\textbf{C-MIB}) \cite{MIB}: It adopts the information bottleneck principle to eliminate redundancy and noise from both unimodal and multimodal features, thus establishing it as a baseline for handling noisy modalities.;
(3) \textbf{Two-stage Modality Denoising and
Complementation} (\textbf{TMDC}) \cite{zhuang2025tmdc}: Its training process involves two distinct stages. The first, the intra-modality denoising stage, aims to enhance representational robustness by using denoising modules to obtain clean modality-specific and shared representations from complete data, thereby reducing noise interference. The second, the inter-modality complementation Stage, utilizes these representations to address modality absence through cross-modal compensation.

The additional baselines for MSA include:

(1)  \textbf{Information-Theoretic Hierarchical Perception} (\textbf{ITHP}) \cite{ithp}:  Its design is grounded in the information bottleneck principle, where one modality is designated as the core, while others act as detectors within the information pathway to distill and refine the information flow;
(2) \textbf{Self-Supervised Multi-task
Multimodal sentiment analysis network} (\textbf{Self-MM}) \cite{mmsa}: This approach employs a self-supervised strategy to infer sentiment labels for individual modalities by leveraging the global labels of multimodal samples, thereby learning more discriminative unimodal representations; 
(3) \textbf{Disentangled-Language-Focused Model} (\textbf{DLF}) \cite{wang2025dlf}: It introduces a feature disentanglement module to separate shared and specific information across modalities. The process is further refined by four geometric measures to reduce redundancy and prioritize language-focused features. A language-targeted attractor is also designed to enhance language representations using complementary information from other modalities; 
(4) \textbf{Gradient and Structure Consistency} (\textbf{GSCon}) \cite{shi2025gradient}: It proposes a balanced gradient direction
that aligns each modality’s optimization direction to ensure
unbiased convergence, and aligns the spatial structure of samples in different modalities to avoid the interaction noise caused by multimodal alignment;
(5) \textbf{Diffusion Bridge} \cite{lee2025diffusion}: It directly reduces the modality gap by leveraging denoising diffusion probabilistic models;
(6) \textbf{Enhanced Dynamic Emotion Experts} (\textbf{EMOE}) \cite{EMOE2025}: The framework comprises a mixture of modality experts that dynamically adjust modality importance based on input features, coupled with a unimodal distillation mechanism to preserve the predictive capacity of individual modalities within the fused representation;
(7) \textbf{Contrastive FEature DEcomposition} (\textbf{ConFEDE})~\cite{confede}: It enhances multimodal representations by jointly performing contrastive learning and contrastive decomposition of features;
(8) \textbf{Kolmogorov–Arnold Network with Multimodal Clean Pareto} \textbf{KAN-MCP}~\cite{KAN-MCP2025}: It combines interpretable cross-modal modeling via KANs with feature denoising and compression based on DRD-MIB, yielding discriminative multimodal inputs while alleviating modality imbalance;
(9) \textbf{Multimodal Adaptation Gate BERT/ALBERT} (\textbf{MAG-BERT/MAG-ALBERT}) \cite{MAG-BERT}: It incorporates visual and acoustic information into BERT/ALBERT through a dedicated multimodal adaptation gate;
(10) \textbf{Acoustic Visual Mix-up Consistent} (\textbf{AV-MC}) \cite{simsv2}: This method utilizes modality mix-up to augment visual and acoustic representations, thereby strengthening their contribution to sentiment analysis;
(11) \textbf{Knowledge-Guided Dynamic Modality Attention Fusion} (\textbf{KUDA}) \cite{kuda}: It guides the multimodal fusion process with external emotional knowledge, dynamically selecting the dominant modality and adjusting the weighting of all modalities;
(12) \textbf{MISA} \cite{MISA}: It decomposes each unimodal input into modality-invariant and modality-specific components, which are subsequently fused for final prediction;
(13) \textbf{MultiModal InfoMax} (\textbf{MMIM}) \cite{MMIM}: In MMIM, representation learning is enhanced by maximizing the mutual information both between unimodal features and between multimodal representations at different levels and their unimodal counterparts;
(14) \textbf{Feature-Disentangled Multimodal Emotion Recognition} (\textbf{FDMER}) \cite{yang2022disentangled_FDMER}: It learns the common and private feature representations for each modality, which achieves the modality consistency and disparity constraints by designing tailored losses for modality-invariant and modality-specific subspaces.

The additional baselines for MHD and MSD include:

(1) \textbf{Multimodal Global Contrastive Learning} (\textbf{MGCL}) \cite{mgcl}: 
MGCL applies supervised contrastive learning to multimodal representations, employing diverse augmentation strategies to construct positive and negative sample pairs for each representation;
(2) \textbf{Multimodal Correlation Learning} (\textbf{MCL}) \cite{mcl}: MCL formulates a supervised correlation learning objective that preserves modality-specific characteristics while fostering a more discriminative joint embedding space;
(3) \textbf{Multimodal Multitask Interaction Learning} (\textbf{MIL}) \cite{MIL}: It performs joint sarcasm and sentiment detection, integrating a cross-modal target attention mechanism to align textual with visual/acoustic content and a multimodal interaction module to model the shared and distinct patterns of both tasks;
(4) \textbf{Decoupled Multimodal Distillation} (\textbf{DMD}) \cite{li2023decoupled_cvpr}: DMD enables adaptive cross-modal knowledge transfer by decoupling unimodal features into modality-irrelevant and modality-exclusive subspaces, followed by a specialized graph distillation unit to handle each subspace effectively;
(5) \textbf{Humor Knowledge Enriched Transformer} (\textbf{HKT}) \cite{HKT}: HKT incorporates humor-centric features as external knowledge to resolve the ambiguity and leverage the subtle sentiment cues within the language modality;
(6) \textbf{Multimodal Transformer} (\textbf{MulT}) \cite{MULT}: MulT leverages stacked cross-modal transformers to project and align source modalities to a target modality, effectively mitigating the inter-modal gap.

\begin{table*}[t]
	\centering
	\caption{ \label{t_u2}Further discussions on different invariance constraints.}
	\resizebox{2.0\columnwidth}{!}{\begin{tabular}{c|c|c|c|c|c|c|c|c|c|c|c}
     \noalign{\hrule height 1pt} 
\rowcolor{lightgray!40}
    &\multicolumn{5}{c|}{CMU-MOSI} & \multicolumn{5}{c|}{CMU-MOSEI} & MUStARD \\
			\hline
            \rowcolor{lightgray!40}
			& Acc7  & Acc2 & F1 score & MAE & Corr & Acc7  & Acc2 & F1 score & MAE & Corr & Acc \\
			\hline
			Standard  & 48.4 & \textbf{90.1} & \textbf{90.1} & \textbf{0.616} & \textbf{0.855} & \textbf{55.2} & 87.4 & 87.3 & 0.523 & 0.781 & 79.7 \\
            Ours  & \textbf{49.8} & 89.6 & 89.5 & \textbf{0.616} & 0.853 & 55.1 & \textbf{87.8} &   \textbf{87.7} & \textbf{0.513} & \textbf{0.793} & \textbf{80.0} \\
			\noalign{\hrule height 1pt}
	\end{tabular}}
 \vspace{-0.3cm}
\end{table*}%

\section{Comparison between Different Invariance Constraints}
\label{sec:comparison_invariant}

In Section~\ref{sec:invariant}, we propose two different methods to realize invariance constraint. The first one is a standard way to compute and minimize the conditional distribution distance (we directly minimize unimodal predictions from different environments for regression tasks, and use KL-divergence to minimize conditional distribution distance for classification tasks), which needs to implement a unimodal predictor for each modality and generate unimodal prediction. The second one is to directly minimize the distance between different causal invariant representations extracted under different environments, which is adopted in our CmIR.  Here we provide a comparison between these two strategies, and the results are shown in Table~\ref{t_u2}. From Table~\ref{t_u2} we can infer that these two strategies actually yield comparable outcomes, likely because although the first method can more directly evaluate the differences between conditional distributions, it introduces a unimodal predictor. However, when the unimodal predictor is insufficiently trained, it may introduce evaluation errors. Therefore, we adopt the second approach, which enables a more rigorous constraint without introducing a unimodal predictor, thereby reducing model complexity.

\begin{table*}[t]
\centering
\small
   \caption{ \label{t_com} Model complexity analysis on the MUStARD and CMU-MOSI datasets.
 } 
\resizebox{\linewidth}{!}{\begin{tabular}{ccc|cccc}
  \noalign{\hrule height 1pt} 
\rowcolor{lightgray!40}
    \multicolumn{3}{c|}{MUStARD} & \multicolumn{4}{c}{CMU-MOSI} \\
    \hline
 \rowcolor{lightgray!40}
 Model  & Acc2 & Parameters (M) &  Model  & Acc2 & Parameters  (M) & FLOPs (G) \\
 \hline
HKT \cite{HKT} & 76.47 & 17.10M & EMOE \cite{EMOE2025} & 84.8 & 317.30M & 10.02\\
MCL \cite{mcl} &  \underline{77.94}  & \underline{13.83M} & GSCon \cite{shi2025gradient} &  \underline{88.1} & 248.52M & 12.22 \\
 MGCL \cite{mgcl} & \underline{77.94} &  14.28M & Diffusion Bridge \cite{lee2025diffusion} & 86.9 & \textbf{185.46M} & 11.52 \\
  \hline
  \rowcolor[HTML]{EBFAFF}
CmIR  & \textbf{80.00} &  \textbf{13.44}  & CmIR & \textbf{89.6} & \underline{187.51M} & 8.95 \\
 \noalign{\hrule height 1pt} 
 \end{tabular}}
\end{table*}%

\section{Model Complexity Analysis}

The proposed CmIR does not design complex multimodal fusion to explore sufficient inter-modal interactions, but merely introduces an encoder and a decoder to learn invariant modality representations for each modality. Each encoder/decoder only has a few layers. Therefore, the model complexity of the proposed CmIR is acceptable. To verify our claim, we compare the model complexity of CmIR with competitive baselines on the MUStARD \cite{msd} and CMU-MOSI \cite{zadeh2016multimodal} datasets. As shown in Table~\ref{t_com}, on MUStARD, \textbf{CmIR outperforms competitive baselines with minimal number of parameters}, demonstrating the effectiveness of the proposed framework and indicating that the performance improvement of CmIR dose not results from the increase in the number of parameters.
On CMU-MOSI, we  evaluate the FLOPs and number of parameters of our CmIR. As shown in Table~\ref{t_com}, CmIR has significantly fewer parameters than EMOE and GSCon, slightly more than Diffusion Bridge, while its FLOPs are lower than all baselines, indicating the high efficiency of our method. 

\begin{table*}[t]
\centering
 \caption{ \label{t_test}\textbf{Paired t-test analysis between CmIR and ITHP \cite{ithp}.} }
\resizebox{1.5\columnwidth}{!}{\begin{tabular}{c|c|c|c|c|c|c|c|c|c|c}
 \noalign{\hrule height 1pt} 
\rowcolor{lightgray!40}
 & \multicolumn{5}{c|}{CMU-MOSI} & \multicolumn{5}{c}{CMU-MOSEI}  \\
  \hline
   &  Acc7  & Acc2  & F1  & MAE  & Corr & Acc7  & Acc2  & F1  & MAE  & Corr  \\
 \hline
ITHP & 0.002 & 0.004 & 0.004 & 4.98e-4 & 0.439 & 3.65e-5 & 5.73e-4 & 7.78e-4 & 1.12e-4 & 0.566 \\
 \noalign{\hrule height 1pt} 
 \end{tabular}}
\end{table*}%

\section{Significant Test}

In this section, we provide the results of significant test between the proposed CmIR and a strong baseline ITHP \cite{ithp} using t-test on the CMU-MOSI and CMU-MOSEI datasets. We run each model for ten times under different random seeds to gather the results.
As shown in Table~\ref{t_test}, the t-test results indicate that CmIR has a statistically significant difference with ITHP \cite{ithp} for all evaluation metrics except the Corr metric ($p<0.05$ denotes a statistically significant difference), suggesting that the improvement of CmIR is significant.

\begin{table*}[t]
\centering
 \caption{ \label{t_prob}\textbf{Probing results on CMU-MOSEI.} }
\resizebox{1.5\columnwidth}{!}{\begin{tabular}{c|c|c|c|c|c|c}
 \noalign{\hrule height 1pt} 
\rowcolor{lightgray!40}
Component & Environment Prediction & \multicolumn{5}{c}{Sentiment prediction}  \\
  \hline
   &  MSE  & Acc7  & Acc2 & F1  & MAE  & Corr   \\
 \hline
 $Z^{inv}$  &  2.04  & 55.1  & 87.8 & 87.7 & 0.513  & 0.793   \\
 \hline
$Z^{spu}$ & 0.87 & 41.4 & 62.8 & 48.5 & 0.841 & 0.001 \\
 \noalign{\hrule height 1pt} 
 \end{tabular}}
\end{table*}%

\section{Probing Experiments for $Z^{inv}$ and $Z^{spu}$}
Orthogonality is only a practical proxy for mutual information minimization. To strengthen the evidence that $Z^{inv}$ and $Z^{spu}$ are well separated and contain different information, we have added two probing experiments: (a) training a predictor to predict the environment label from $Z^{inv}$ or $Z^{spu}$ alone. Since the difference between different environments lies in the noise level (noise coefficient $\alpha^{(e)}$), we directly train an additional fusion network and predictor that use $Z^{inv}$ / $Z^{spu}$ to predict the noise coefficient $\alpha^{(e)}$. Finally, we use MSE to evaluate the effectiveness of the prediction (lower accuracy from $Z^{inv}$ indicates better invariance); (b) Predicting the sentiment label from each component separately. As shown in Table~\ref{t_prob}, $Z^{inv}$ yields high MSE in environment prediction but high accuracy in sentiment prediction, while $Z^{spu}$ does the opposite, confirming semantic separation. Notably, $Z^{spu}$ only relies on label bias for prediction (predicting all samples as positive yields an accuracy of 62.8\% which is exactly the accuracy of the prediction from $Z^{spu}$ ), suggesting it contains non-causal information.

\begin{table*}[t]
\centering
 \caption{ \label{t_sens}\textbf{Sensitivity Analysis of Noise Intensity $\alpha^{(1)}$ on CMU-MOSI (OOD).} }
\resizebox{1.5\columnwidth}{!}{\begin{tabular}{c|c|c|c|c|c|c|c|c|c}
 \noalign{\hrule height 1pt} 
\rowcolor{lightgray!40}
$\alpha^{(1)}$ & 1e-8 & 1e-7 & 1e-6 & 1e-5 & 1e-4 & 1e-3 & 1e-2 & 1e-1 & 1  \\
  \hline
 Acc7  &  43.5 & 45.8 & 47.5 & 47.0 & 47.3 & 48.0 & 48.8 & 48.5 & 49.0 \\
 \hline
Acc2 & 79.7 & 82.8 & 84.1 & 83.6 & 84.0 & 84.4 & 83.9 & 84.4 & 84.6 \\
\hline
F1 score & 79.7 & 82.7 & 83.9 & 83.5 & 84.1 & 84.3 & 83.9 & 84.4 & 84.6 \\
 \noalign{\hrule height 1pt} 
 \end{tabular}}
\end{table*}%

\section{Sensitivity Analysis of Noise Intensity}
In this section, we have conducted new experiments on CMU-MOSI (OOD) and provided the sensitivity analysis of noise intensity $\alpha^{(1)}$ in Table~\ref{t_sens}. 
As shown in Table~\ref{t_sens}, when the value of $\alpha^{(1)}$ is small, CmIR performs poorly. This is because when the imposed noise is small, its impact on modality representations is limited, making it difficult to simulate diverse environments for learning effective causal invariant representations. Moreover, when $\alpha^{(1)}$ fluctuates within a large range (from 1e-7 to 1), the model maintains good performance, demonstrating the robustness of CmIR. Additionally, when $\alpha^{(1)}=1$, the model achieves stronger performance than the default setting ($\alpha^{(1)}=0.1$), indicating the performance potential of CmIR (better results can be obtained through more careful hyperparameter tuning). Compared with the number of environments $K$ (see Figure~\ref{9_mosi} (d)), $\alpha^{(1)}$ has a greater impact and should be set to a relatively large value.




\end{document}